\documentclass[11pt]{article} 

\pdfoutput=1
\usepackage{enumerate}
\usepackage{pdfsync}
\usepackage[OT1]{fontenc}
\usepackage{fullpage}

\usepackage[colorlinks,
            linkcolor=red,
            anchorcolor=blue,
            citecolor=blue
            ]{hyperref}
\usepackage[protrusion=true,expansion=true]{microtype}
\usepackage{pbox}
\usepackage{setspace}
\usepackage{float}
\usepackage{wrapfig,lipsum}
\usepackage{enumitem}
\usepackage{graphicx}
\usepackage{enumerate}
\usepackage{enumitem}
\usepackage{pgfplots}
\usepackage{smile}
\usetikzlibrary{arrows,shapes,snakes,automata,backgrounds,petri}
\usepackage{booktabs} 
\usepackage[T1]{fontenc}
\usepackage{caption}

\newcommand{\la}{\langle}
\newcommand{\ra}{\rangle}
\newcommand{\qvalue}{Q}
\newcommand{\vvalue}{V}

\newcommand{\reward}{r}

\newcommand{\regret}{\mathrm{Regret}}

\newcommand{\conf}{\mathrm{conf}}
\newcommand{\var}{\mathrm{var}}

\usepackage{colortbl}
\definecolor{LightCyan}{rgb}{0.8, 0.9, 1}
\allowdisplaybreaks
\usepackage{xcolor}
\makeatletter
\newcommand*{\rom}[1]{\expandafter\@slowromancap\romannumeral #1@}
\makeatother
\title{\huge Variance-Dependent Regret Bounds for Linear Bandits and Reinforcement Learning: Adaptivity and Computational Efficiency}
\date{}
\begin{small}
\author
{
    Heyang Zhao\thanks{Department of Computer Science, University of California, Los Angeles, CA 90095, USA; e-amil:{\tt hyzhao@cs.ucla.edu}}
    ~~
    Jiafan He\thanks{Department of Computer Science, University of California, Los Angeles, CA 90095, USA; e-mail: {\tt jiafanhe19@ucla.edu}}
    ~~
	Dongruo Zhou\thanks{Department of Computer Science, University of California, Los Angeles, CA 90095, USA; e-mail: {\tt drzhou@cs.ucla.edu}} 
	~~
    Tong Zhang \thanks{‡Google Research and The Hong Kong University of Science and Technology; e-mail:{\tt tongzhang@tongzhang-ml.org}}
    ~~
    Quanquan Gu\thanks{Department of Computer Science, University of California, Los Angeles, CA 90095, USA; e-mail: {\tt qgu@cs.ucla.edu}}
}\end{small}

\begin{document}

\maketitle
\begin{abstract}%


    Recently, several studies \citep{zhou2021nearly, zhang2021variance, kim2021improved, zhou2022computationally} have provided variance-dependent regret bounds for linear contextual bandits, which interpolates the regret for the worst-case regime and the deterministic reward regime. However, these algorithms are either computationally intractable or unable to handle unknown variance of the noise. 
      In this paper, we present a novel solution to this open problem by proposing the \emph{first computationally efficient} algorithm for linear bandits with heteroscedastic noise. Our algorithm is adaptive to the unknown variance of noise and achieves an $\tilde{O}(d \sqrt{\sum_{k = 1}^K \sigma_k^2} + d)$  regret, where $\sigma_k^2$ is the \emph{variance} of the noise at the round $k$, $d$ is the dimension of the contexts and $K$ is the total number of rounds. Our results are based on an adaptive variance-aware confidence set enabled by a new Freedman-type concentration inequality for self-normalized martingales and a multi-layer structure to stratify the context vectors into different layers with different uniform upper bounds on the uncertainty. 
      
      Furthermore, our approach can be extended to linear mixture Markov decision processes (MDPs) in reinforcement learning. We propose a variance-adaptive algorithm for linear mixture MDPs, which achieves a problem-dependent horizon-free regret bound that can gracefully reduce to a nearly constant regret for deterministic MDPs. Unlike existing nearly minimax optimal algorithms for linear mixture MDPs, our algorithm does not require explicit variance estimation of the transitional probabilities or the use of high-order moment estimators to attain horizon-free regret. We believe the techniques developed in this paper can have independent value for general online decision making problems.

    \end{abstract}

    
    \section{Introduction}
    
    The Multi-Armed Bandits (MAB) problem has been persistently studied since 1933 \citep{thompson1933likelihood,robbins1952some, cesa1998finite, auer2002finite}. In the past decades, a variety of bandit algorithms have been developed under different settings, emerging their practicality in assorted real world tasks such as online advertising \citep{li2010contextual}, clinical experiments \citep{villar2015multi} and resource allocations \citep{lattimore2015linear}, to mention a few. For a thorough review of bandit algorithms, please refer to \citet{bubeck2012regret,lattimore2020bandit}.
    
    
    To deal with an excessive number of arms, contextual linear bandits \citep{auer2002using,abe2003reinforcement,li2010contextual}, where each arm is associated with a context vector and the expected reward is a linear function of the context vector, have garnered a lot of attention. Numerous studies have attempted to design algorithms to achieve the optimal regret bound for linear bandits \citep{chu2011contextual,AbbasiYadkori2011ImprovedAF}. Despite the achievement of minimax-optimal regret bounds in various settings, they only quantify the performance of a specific algorithm under the worst-case scenario. 
    However, in the noiseless scenario (i.e., the variance of the noise equal $0$), the learner only requires $\tilde{O}(d)$ regret to recover the underlying  coefficients of the linear function. This motivates a series of work on variance-dependent regret for linear bandits \citep{zhou2021nearly, zhang2021variance, zhou2022computationally, zhao2022bandit}, which bridges the gap between the worst-case constant-variance regime (i.e., noisy case) and the deterministic regime (i.e., noiseless case). In these works, the regret bounds depend on the variance of noise at each round, i.e., $\{\sigma_i^2\}_{i = 1}^K$ where $K$ is the total number of rounds. Unfortunately, all these prior approaches are either computationally inefficient or non-adaptive, meaning the agent must possess prior knowledge of the variance to learn the reward function. As a result, none of  the existing algorithms are practical enough for real-world use, despite being designed for better performance in reality. Therefore, an open question arises: \begin{center}
    \emph{Can we design computationally tractable algorithms for linear bandits with heteroscedastic noise to obtain a variance-dependent regret bound without prior knowledge on the noise?}
    \end{center}
    
    \subsection{Our Contributions}
    
    
    In this paper, we answer this question affirmatively by proposing the first computationally efficient algorithm for heteroscedastic linear bandits with unknown variance and attaining a regret bound scales as
    $
    \tilde{O}\big(d\sqrt{\textstyle{\sum_{k = 1}^K} \sigma_k^2} + d\big),$
    where $\sigma_k^2$ is the \emph{variance} of the noise at the round $k$, $d$ is the dimension of the contexts and $K$ is the total number of rounds. Our result is significant in the sense that it is minimax optimal in both the deterministic case and the worst case. When there is no noise, the above regret degenerates to $\tilde{O}(d)$, which corresponds to the benign regime. In the worst case when the noise is $R$-sub-Gaussian, the above regret reduces to $\tilde{O}(dR\sqrt{K} + d)$, which matches the minimax regret bound proved in \citet{AbbasiYadkori2011ImprovedAF}. Please refer to Table \ref{table:1} for a comparison between our result and the previous results in linear contextual bandits. 
    
    
    \newcolumntype{g}{>{\columncolor{LightCyan}}c}
    \begin{table*}[t!]
    \caption{Comparison between different algorithms for stochastic linear contextual bandits. $d$, $K$, $\{\sigma_k\}_{k \in [K]}$ are the dimension of context vectors, the number of rounds and the variance of noise at round $k \in [K]$. The last column indicates whether the corresponding algorithm requires the variance information to achieve variance-dependent regret. } \label{table:1}
    \centering
    \resizebox{\columnwidth}{!}{%
    \begin{tabular}{ggggg}
    \toprule 
    \rowcolor{white} Algorithm & Regret (General-Case) & Regret  (Deterministic-Case)\footnotemark[1] &  Efficiency  & Variances  \\
    \midrule
    \rowcolor{white} $\text{ConfidenceBall}_2$ & & & & \\
    \rowcolor{white} \small{\citep{dani2008stochastic}}  & \multirow{-2}{*}{$\tilde O(d\sqrt{K})$} & \multirow{-2}{*}{$\tilde O(d\sqrt{K})$} & \multirow{-2}{*}{Yes} & \multirow{-2}{*}{N/A}\\ 
    \rowcolor{white} OFUL & & & & \\
    \rowcolor{white} \small{\citep{AbbasiYadkori2011ImprovedAF}}  & \multirow{-2}{*}{$\tilde O(d\sqrt{K})$} & \multirow{-2}{*}{$\tilde O(d\sqrt{K})$} & \multirow{-2}{*}{Yes} & \multirow{-2}{*}{N/A}\\
    \rowcolor{white} Weighted OFUL & &  & &\\ 
    \rowcolor{white} \small{\citep{zhou2021nearly}}  & \multirow{-2}{*}{$\tilde O\Big(d\sqrt{\sum_{k=1}^K \sigma_k^2}+\sqrt{dK}\Big)$} & \multirow{-2}{*}{$\tilde O(\sqrt{dK})$} & \multirow{-2}{*}{Yes} & \multirow{-2}{*}{Known}\\
    \rowcolor{white} Weighted OFUL+ & &  & &\\   
    \rowcolor{white} \small{\citep{zhou2022computationally}}  & \multirow{-2}{*}{$\tilde O\Big(d\sqrt{\sum_{k=1}^K \sigma_k^2}+d\Big)$} & \multirow{-2}{*}{$\tilde O(d)$} & \multirow{-2}{*}{Yes} & \multirow{-2}{*}{Known}\\
    \rowcolor{white} VOFUL & & & & \\ 
    \rowcolor{white}\small{\citep{zhang2021variance}}& \multirow{-2}{*}{$\tilde O\Big(d^{4.5}\sqrt{\sum_{k=1}^K \sigma_k^2}+d^5\Big)$} & \multirow{-2}{*}{$\tilde O(d^2)$} & \multirow{-2}{*}{No}& \multirow{-2}{*}{Unknown}\\
    \rowcolor{white} VOFUL2 & & & & \\ 
    \rowcolor{white}\small{\citep{kim2021improved}}& \multirow{-2}{*}{$\tilde O\Big(d^{1.5}\sqrt{\sum_{k=1}^K \sigma_k^2}+d^2\Big)$} & \multirow{-2}{*}{$\tilde O(d^5)$} & \multirow{-2}{*}{No}& \multirow{-2}{*}{Unknown}\\
    \texttt{SAVE} & & & & \\
    \small{(Theorem \ref{thm:regret1})}& \multirow{-2}{*}{$\tilde O\Big(d\sqrt{\sum_{k=1}^K \sigma_k^2}+d\Big)$}& \multirow{-2}{*}{$\tilde O(d)$} & \multirow{-2}{*}{Yes}& \multirow{-2}{*}{Unknown}\\
    \bottomrule
    \end{tabular}
    }
    
    \end{table*}
    \footnotetext[1]{For the deterministic-case, the variance at stage $k\in[K]$ satisfies $\sigma_k=0$. The regret guarantee is the same as the general case for variance-unaware algorithms.}

    Our algorithm and its analysis rely on the following new techniques.
    \begin{itemize}[leftmargin =*]
        \item We propose a new Freedman-type concentration inequality for vector-valued self-normalized martingales, which is applicable to the heteroscedastic random variables. This strictly generalizes the previous Bernstein-type concentration inequality (Theorem 4.1, \citealt{zhou2021nearly}) for vector-valued self-normalized martingales with homoscedastic random variables. 
        \item  We employ a multi-layer structure to partition the observed context vectors according to their elliptical norm. Different from the classic \texttt{SupLinUCB} algorithm \citep{chu2011contextual}, we use carefully designed weights within each layer to ensure that all the reweighted context vectors in the same layer have a uniform elliptical norm.   
        \item Equipped with the new concentration inequality, we design a new adaptive variance-aware exploration strategy. Specifically, we adopt a self-adaptive confidence set whose radius is updated at each round based on the `square loss' incurred by the online estimator. 
    \end{itemize}

    Furthermore, we apply our novel techniques to episodic Markov decision processes, where the agent interacts with the environment by taking actions and observing states and rewards generated by the unknown dynamics over time. We focus on linear mixture MDPs \citep{jia2020model,ayoub2020model,zhou2021provably} in this paper, whose transition dynamic is assumed to be a linear combination of $d$ basic transition models. For linear mixture MDPs, both minimax optimal horizon-dependent regret \citep{zhou2021nearly} and horizon-free regret \citep{zhang2021variance, kim2021improved, zhou2022computationally} have been achieved. We propose an algorithm named \texttt{UCRL-AVE} and derive a tight problem-dependent regret bound that has no explicit polynomial dependency on neither the number of episodes $K$ nor the planing horizon $H$. Our regret bound gracefully degrades to the nearly minimax optimal horizon-free regret bound achieved by \citet{zhou2022computationally} in the worst case. See Table \ref{table:2} for a comparison between our regret bound with the previous results regarding linear mixture MDPs. 
    


    \newcolumntype{g}{>{\columncolor{LightCyan}}c}
    \begin{table*}[t!]
    \caption{Comparison of our variance-dependent regret with existing regret bounds for linear mixture MDPs. $H$, $d$, $K$, are the horizon of the underlying MDP, the dimension of the feature vectors and the number of episodes. $\Var_K^*$ is defined in Section \ref{sec-4} to characterize the randomness of MDPs. It is shown later in Section \ref{sec-4} that our variance-dependent regret degrades to $\tilde{O}(d\sqrt{K} + d^2)$ in the worst case, which matches the nearly minimax optimal horizon-free regret in linear mixture MDPs. }\label{table:2}
    \centering 
    \resizebox{\columnwidth}{!}{%
    \begin{tabular}{ggggg}
    \toprule 
    \rowcolor{white} Algorithm & Regret (General-Case) & Variance-Dependent &   Assumption & Efficiency  \\
    \midrule  
    \rowcolor{white} $\text{UCRL-VTR}$ & & & Homogeneous & \\
    \rowcolor{white} \small{\citep{ayoub2020model,jia2020model}}  & \multirow{-2}{*}{$\tilde O(d\sqrt{H^3K})$} & \multirow{-2}{*}{No} & {$\sum_{h=1}^{H} r_h \leq H$} & \multirow{-2}{*}{Yes}\\ 
    \rowcolor{white} UCRL-VTR+ & $\tilde O(\sqrt{d^2H^3+dH^4}\sqrt{K}$ & & Inhomogeneous & \\
    \rowcolor{white} \small{\citep{zhou2021nearly}}  & {$+d^2H^3+d^3H^2)$} & \multirow{-2}{*}{No} & {$\sum_{h=1}^{H} r_h \leq H$}  & \multirow{-2}{*}{Yes}\\
    \rowcolor{white} VarLin & &  & Homogeneous &\\ 
    \rowcolor{white} \small{\citep{zhang2021variance}}  & \multirow{-2}{*}{$\tilde O\Big(d^{4.5}\sqrt{K}+d^9\Big)$} & \multirow{-2}{*}{No} & {$\sum_{h=1}^{H} r_h \leq 1$} & \multirow{-2}{*}{No}\\
    \rowcolor{white} VarLin2 & &  & Homogeneous &\\   
    \rowcolor{white} \small{\citep{kim2021improved}}  & \multirow{-2}{*}{$\tilde O(d\sqrt{K}+d^2)$} & \multirow{-2}{*}{No}  & {$\sum_{h=1}^{H} r_h \leq 1$} & \multirow{-2}{*}{No}\\
    \rowcolor{white} HF-UCRL-VTR+ & & &  Homogeneous & \\ 
    \rowcolor{white}\small{\citep{zhou2022computationally}}& \multirow{-2}{*}{$\tilde O(d\sqrt{K}+d^2)$} & \multirow{-2}{*}{No}& {$\sum_{h=1}^{H} r_h \leq 1$} & \multirow{-2}{*}{Yes}\\
    \texttt{UCRL-AVE} & & & Homogeneous & \\
    \small{(Theorem \ref{thm:regret1})}& \multirow{-2}{*}{$\tilde O\Big(d \sqrt{\Var_K^*} + d^{2}\Big)$}& \multirow{-2}{*}{Yes} & {$\sum_{h=1}^{H} r_h \leq 1$} & \multirow{-2}{*}{Yes}\\
    \bottomrule 
    \end{tabular}
    }
    \end{table*}

    \subsection{Related Work}
    
    \noindent \textbf{Problem-dependent regret in RL. }
    Most of the performance guarantees for episodic MDPs have been focused on worst-case regret bounds. However, some works have achieved problem-dependent regret, which holds in various scenarios, as demonstrated by studies such as \citet{zanette2019tighter, Simchowitz2019NonAsymptoticGR, jin2020reward, he2021logarithmic,Dann2021BeyondVG, xu2021fine, wagenmaker2022first}. These results can be broadly categorized into two groups. The first group is first-order regret in RL, which was originally proposed by \citet{zanette2019tighter} and later extended to the linear MDP setting by \citet{wagenmaker2022first}. The second group is gap-dependent regret guarantees, which have been studied for both tabular MDPs \citep{Simchowitz2019NonAsymptoticGR, xu2021fine, Dann2021BeyondVG} and linear MDPs/linear mixture MDPs \citep{he2021logarithmic}. We also notice that a concurrent work by \citet{zhou2023sharp} considers variance-dependent bound in tabular MDPs. Our paper utilizes the same definition of \emph{trajectory-based total variance} as \citet{zhou2023sharp}, which characterizes the \emph{randomness} of an episodic MDP. 

    \noindent\textbf{Horizon-free regret in tabular RL. }
    RL is considered to be more challenging than contextual bandits due to the non-trivial planning horizon and uncertain state transitions. \citet{jiang2018open} conjectured that any algorithm seeking an $\epsilon$-optimal policy for tabular RL, where the total reward is bounded by $1$, would require a sample complexity with a polynomial dependence on the planning horizon $H$. However, this conjecture was disproven by \citet{wang2020long}, who introduced a horizon-free algorithm with a sample complexity of $\tilde O(|\cS|^5|\cA|^4\epsilon^{-2}\text{polylog}(H))$ that only has a polylogarithmic dependence on $H$. \citet{zhang2021reinforcement} then proposed a near-optimal algorithm with a regret of $O((\sqrt{|\cS||\cA|K}+|\cS|^2|\cA|)\text{polylog}(H))$ and a similar sample complexity. Later, \citet{li2022settling} and \citet{zhang2022horizon} presented algorithms with sample complexity guarantees that are independent of $H$. 

    \noindent \textbf{Heteroscedastic linear bandits. }
    The worst-case regret of linear bandits has been extensively studied \citep{auer2002using, dani2008stochastic, li2010contextual, chu2011contextual, AbbasiYadkori2011ImprovedAF, li2019nearly}. Recently, there is a series of works considering a heteroscedastic variant of the classic linear bandit problem where the noise distribution is assumed to vary over time. \citet{kirschner2018information} is the first to formally propose linear bandit model with heteroscedastic noise. In their setting, the noise at round $k \in [K]$ is assumed to be $\sigma_k$-sub-Gaussian. Some recent works relaxed the sub-Gaussian assumption in the sense that the noise at the $k$-th round is assumed to be of variance $\sigma_k^2$ instead of $\sigma_k$-sub-Gaussian \citep{zhou2021nearly, zhang2021variance, kim2021improved, zhou2022computationally, dai2022variance}. Among these works, \citet{zhou2021nearly} and \citet{zhou2022computationally} considered known-variance case where $\sigma_k$ is observed by the learner after the $k$-th round, while \citet{zhang2021variance, kim2021improved} proposed statistically efficient but computationally inefficient algorithms for the unknown-variance case. \citet{dai2022variance} considered a more specific model, heteroscedastic sparse linear bandits, and proposed a general framework which converts any heteroscedastic linear bandit algorithm to an algorithm for heteroscedastic sparse linear bandits. 
    
    \noindent \textbf{RL with linear function approximation. }
    There is a huge body of literature on RL with linear function approximation \citep{jiang2017contextual, dann2018oracle, yang2019sample, jin2020provably, wang2020optimism, du2019good, sun2019model, zanette2020frequentist, zanette2020learning, weisz2021exponential, yang2020reinforcement, modi2020sample, ayoub2020model, zhou2021nearly, he2021logarithmic, zhou2022computationally}. Several types of assumptions on the linear structure of the underlying MDPs have been made in these works, including the \emph{linear MDP} assumption \citep{yang2019sample, jin2020provably, hu2022nearly, he2022nearly, agarwal2022vo}, the \emph{low Bellman-rank} assumption \citep{jiang2017contextual}, the \emph{low inherent Bellman error} assumption \citep{zanette2020learning}, and the \emph{linear mixture MDP} assumption \citep{modi2020sample,jia2020model, ayoub2020model, zhou2021nearly}. In this paper, we focus on linear mixture MDPs, where the transition probability function is assumed to be a linear function of a known feature mapping over the state-action-next-state triplet. Recently, there is a line of works attaining horizon-free regret bounds \citep{zhang2021variance, kim2021improved, zhou2022computationally} for linear mixture MDPs, which are most related to our work.

    \paragraph{Notation} 
    We use lower case letters to denote scalars, and use lower and upper case bold face letters to denote vectors and matrices respectively. We denote by $[n]$ the set $\{1,\dots, n\}$. For a vector $\xb\in \RR^d$ and a positive semi-definite matrix $\bSigma\in \RR^{d\times d}$, we denote by $\|\xb\|_2$ the vector's Euclidean norm and define $\|\xb\|_{\bSigma}=\sqrt{\xb^\top\bSigma\xb}$. For $\xb, \yb\in \RR^d$, let $\xb\odot\yb$ be the Hadamard (componentwise) product of $\xb$ and $\yb$. For two positive sequences $\{a_n\}$ and $\{b_n\}$ with $n=1,2,\dots$, 
    we write $a_n=O(b_n)$ if there exists an absolute constant $C>0$ such that $a_n\leq Cb_n$ holds for all $n\ge 1$ and write $a_n=\Omega(b_n)$ if there exists an absolute constant $C>0$ such that $a_n\geq Cb_n$ holds for all $n\ge 1$. We use $\tilde O(\cdot)$ to further hide the polylogarithmic factors. We use $\ind\{\cdot\}$ to denote the indicator function. For $a,b \in \RR$ satisfying $a \leq b$, we use 
    $[x]_{[a,b]}$ to denote the truncation function $x\cdot \ind\{a \leq x \leq b\} + a\cdot \ind\{x<a\} + b\cdot \ind\{x>b\}$.

    \section{Variance-Aware Learning for Heteroscedastic Linear Bandits} \label{sec:bandits}
    
    In this section, we propose a computationally efficient variance-aware algorithm, dubbed \texttt{SAVE} (\textbf{S}uplin + \textbf{A}daptive \textbf{V}ariance-aware \textbf{E}xploration), for stochastic linear contextual bandits and present its theoretical guarantees. \texttt{SAVE} does not require the knowledge of the noise variance (or its upper bound), making it adaptable to varying levels of noise variance.
    
    
    \subsection{Problem Setup}
    
    We consider a heteroscedastic variant of the classic linear contextual bandit problem \citep{zhou2021nearly,zhang2021variance}. Let $K$ be the total  number of rounds. At each round $k\in[K]$, the interaction between the agent and the environment  is as follows: 
        (1) the environment generates an arbitrary decision set $\cD_k \subseteq \RR^d$ where each element represents a feasible action that can be selected by the agent; (2) the agent observes $\cD_k$ and selects $\ab_k \in \cD_k$; and (3) the environment generates the stochastic noise $\epsilon_k$ at round $k$ and reveal the stochastic reward $r_k = \la \btheta^*, \ab_k \ra + \epsilon_k$ to the agent. We assume that there exists a uniform bound $A > 0$ for the $\ell_2$ norm of the feasible actions, i.e., for all $k \in [K]$, $\ab \in \cD_k$, it holds that $\|\ab\|_2 \le A$. Following \citet{zhou2021nearly}, we assume the following condition on the random noise $\epsilon_k$ at each round $k$: 
    \begin{align} 
        \PP\left(|\epsilon_k| \le R\right) = 1,\quad \EE[\epsilon_k | \ab_{1:k}, \epsilon_{1:k - 1}] = 0,\quad\EE[\epsilon_k^2 | \ab_{1:k}, \epsilon_{1:k - 1}] = \sigma_k^2. \label{eq:noise:condition}
    \end{align}
    Without loss of generality, we assume that the size of $\cD_k$ is finite and is bounded by $|\cD|$ for all $k \in [K]$. If the size of $\cD_k$ is infinite, we can use standard covering argument to convert it to be finite. 
    
    The goal of the agent is to minimize the cumulative regret defined as follows: \begin{align} 
        \regret(K) = \sum_{k \in [K]} \big(\la \ab_k^*, \btheta^*\ra - \la \ab_k, \btheta^*\ra\big), \quad \text{where}\  \ab_k^* = \argmax_{\ab \in \cD_k} \la \ab, \btheta^* \ra. \label{eq:def:regret}
    \end{align}
    \subsection{Technical Challenges} \label{subsec:challenges}
    
    The key technical challenge we face is to provide a tight upper bound of $\la \ab_k^*, \btheta^*\ra - \la \ab_k, \btheta^*\ra$. The classical approach is to use the \emph{optimism-in-the-face-of-uncertainty} principle \citep{AbbasiYadkori2011ImprovedAF}, and construct a confidence set $\cC_k$ which includes $\btheta^*$ w.h.p., then upper bound $\la \ab_k^*, \btheta^*\ra$ with $\la \ab_k, \btheta_k\ra$, where $\btheta_k \in \cC_k$. Starting from here, there are two main approaches to bound $\la\ab_k, \btheta_k - \btheta^* \ra$ for heteroscedastic linear bandits.

    The first approach bounds $\la\ab_k, \btheta_k - \btheta^* \ra$ with $\|\ab_k\|_{\hat\bSigma_k^{-1}} \|\btheta_k - \btheta^*\|_{\hat\bSigma_k}$ by Cauchy-Schwarz inequality.  \citet{zhou2021nearly, zhou2022computationally} constructed $\cC_k$ as an ellipsoid centering at $\hat\btheta_k$, which is the solution to a \emph{weighted linear regression} problem over the past contexts $\ab_i$, and their weight is based on the upper bound of the variance of heteroscedastic noise $\sigma_k^2$. Then they bound $ \|\btheta_k - \btheta^*\|_{\hat\bSigma_k}$ by $\|\btheta_k - \hat\btheta_k\|_{\hat\bSigma_k}$ and $\|\btheta^* - \hat\btheta_k\|_{\hat\bSigma_k}$ separately,  each of which can be bounded properly by using the self-normalized concentration inequalities proposed in \citet{zhou2021nearly, zhou2022computationally}. However, as we have mentioned before, their approach is limited to the case where an upper bound of $\sigma_k^2$ is known.


    The second approach \citep{zhang2021variance, kim2021improved} follows the \emph{test-based} framework. Instead of constructing $\cC_k$ as an ellipsoid centering at a least square estimator $\hat{\btheta}_k$ for each round $k$, \citet{zhang2021variance, kim2021improved} constructed $\cC_k$ as the intersection of a series of sub-confidence sets denoted by different tests, where each test represents a constraint over a potential direction of the action $\ab_k$. The limitation of their approach is that, in order to have a uniform upper bound on $\la\ab_k, \btheta_k - \btheta^* \ra$, they have to cover all possible directions of $\ab_k$, which leads to an $\exp(d)$ number of test sets by the standard covering argument. This makes the computational time of their test-based algorithms depend on $d$ exponentially, which is computationally inefficient.



    \subsection{A New Freedman-Type Concentration Inequality for Vector-Valued Martingales} \label{subsec:bernstein}
    
    To tackle the above technical challenges posed by both weighted linear regression and test-based approach, we seek to develop a new Freedman-type concentration inequality for vector-valued self-normalized martingales with heteroscedastic noise (i.e., non-uniform variance). 
    \begin{theorem}\label{thm:bernstein1}
        Let $\{\cG_k\}_{k=1}^\infty$ be a filtration, and $\{\xb_k,\eta_k\}_{k\ge 1}$ be a stochastic process such that
        $\xb_k \in \RR^d$ is $\cG_k$-measurable and $\eta_k \in \RR$ is $\cG_{k+1}$-measurable.
        Let $L,\sigma,\lambda, \epsilon>0$, $\bmu^*\in \RR^d$. 
        For $k\ge 1$, 
        let $y_k = \la \bmu^*, \xb_k\ra + \eta_k$, where $\eta_k, \xb_k$ satisfy 
        \begin{align}
            \EE[\eta_k|\cG_k] = 0, \ |\eta_k| \leq R, \ \sum_{i = 1}^k \EE[\eta_i^2|\cG_i] \le v_k,  \ \ \text{for}\ \forall \ k\geq 1 \notag
        \end{align}
        For $k\ge 1$, let $\Zb_k = \lambda\Ib + \sum_{i=1}^{k} \xb_i\xb_i^\top$, $\bbb_k = \sum_{i=1}^{k}y_i\xb_i$, $\bmu_k = \Zb_k^{-1}\bbb_k$, and
        \begin{align}
            \beta_k &= 16\rho \sqrt{v_k \log(4k^2 / \delta)} + 6 \rho R  \log(4k^2 / \delta), \notag
        \end{align}
        where $\rho \ge \sup_{k \ge 1}\|\xb_k\|_{\Zb_{k - 1}^{-1}}$. 
    Then, for any $0 <\delta<1$, we have with probability at least $1-\delta$ that, 
        \begin{align}
            \forall k\geq 1,\ \big\|\textstyle{\sum}_{i=1}^{k} \xb_i \eta_i\big\|_{\Zb_k^{-1}} \leq \beta_k,  \|\bmu_k - \bmu^*\|_{\Zb_k} \leq \beta_k + \sqrt{\lambda}\|\bmu^*\|_2. \notag
        \end{align}
    \end{theorem}

    Theorem \ref{thm:bernstein1} can be viewed as an extension of Freedman's inequality \citep{Freedman1975OnTP} from scalar-valued martingales to vector-valued self-normalized martingales. 
    Though \citet{zhou2021nearly} proposed Bernstein-type concentration inequalities for vector-valued martingales (Theorem 4.1, \citealt{zhou2021nearly}), their inequality relies on a uniform upper bound on the variance of random variables, i.e., $\|\bmu_k - \bmu^*\|_{\Zb_k} \le \tilde{O}(\sigma\sqrt{d} + R)$, where $\sigma^2 \ge \sup_{k \ge 1} \EE[\eta_i^2|\cG_i]$. In contrast, the upper bound provided by Theorem \ref{thm:bernstein1} depends on the maximum of $\|\xb_k\|_{\Zb_{k-1}^{-1}}$, which is of the order $\tilde{O}(\sqrt{d/k})$ under certain conditions \citep{carpentier2020elliptical}. For these cases, our upper bound for $\|\bmu_k - \bmu^*\|_{\Zb_k}$ scales as $\tilde{O}(\sqrt{d \cdot v_k/k}  + R \cdot  \sqrt{d / k})$, 
    which is more fine-grained and strictly tighter than the previous upper bounds when $k \ge d$. 

    \subsection{The Proposed Algorithm} \label{subsec:algorithm1}
    
    Equipped with the new Freedman-type concentration inequality, we will design a new algorithm that is adaptive to the unknown variance of noise. 
    
    \begin{algorithm*}[t!] 
        \caption{SupLin + Adaptive Variance-aware Exploration (SAVE)} \label{alg:1}
        \begin{algorithmic}[1]
            \REQUIRE $\alpha > 0$, and the upper bound on the $\ell_2$-norm of $\ab$ in $\cD_k (k\ge 1)$, i.e., $A$. 
        \STATE Initialize $L \leftarrow \lceil \log_2 (1 / \alpha) \rceil$. 
        \STATE Initialize the estimators for all layers: $\hat\bSigma_{1, \ell} \leftarrow 2^{-2\ell} \cdot \Ib$, $\hat\bbb_{1, \ell} \leftarrow \zero$, $\hat\btheta_{1, \ell} \leftarrow \zero$, $\Psi_{k, \ell} \leftarrow \varnothing$, $\hat \beta_{1, \ell} \leftarrow 2^{-\ell + 1}$ for all $\ell \in [L]$. 
        \FOR{$k=1,\ldots, K$}
        \STATE Observe $\cD_k$. 
        \STATE Let $\cA_{k, 1} \leftarrow \cD_k$, $\ell \leftarrow 1$. \label{alg:1:line:5}
        \WHILE{$\ab_k$ is not specified}
        \IF{$\|\ab\|_{\hat\bSigma_{k, \ell}^{-1}} \le \alpha$ for all $\ab \in \cA_{k, \ell}$}
            \STATE Choose $\ab_k \leftarrow \argmax_{\ab \in \cA_{k, \ell}} \la \ab, \hat{\btheta}_{k, \ell}\ra + \hat\beta_{k, \ell} \|\ab\|_{\hat\bSigma_{k, \ell}^{-1}}$ and observe $r_k$. \label{line:selection} 
            \STATE Keep the same index sets at all layers: $\Psi_{k + 1, \ell'} \leftarrow \Psi_{k, \ell'}$ for all $\ell' \in [L]$. \label{alg1:line:10}
        \ELSIF{$\|\ab\|_{\hat\bSigma_{k, \ell}^{-1}} \le 2^{-\ell}$ for all $\ab \in \cA_{k, \ell}$} \label{alg1:line:11}
            \STATE $\cA_{k, \ell + 1} \leftarrow \big\{\ab \in \cA_{k, \ell} \big| \la \ab, \hat\btheta_{k, \ell} \ra \ge \max_{\ab' \in \cA_{k, \ell}} \la \ab', \hat\btheta_{k, \ell} \ra - 2 \cdot 2^{-\ell} \hat\beta_{k, \ell}\big\}$. \label{alg1:line:12}
        \ELSE \label{alg1:line:13}
        \STATE Choose $\ab_k$ such that $\|\ab_k\|_{\hat\bSigma_{k, \ell}^{-1}} > 2^{-\ell}$ and observe $r_k$. \label{alg1:line:14}
        \STATE Compute the weight: $w_k \leftarrow 2^{-\ell}/{\|\ab_k\|_{\hat\bSigma_{k, \ell}^{-1}}}$. \label{alg:1:line:15} 
        \STATE Update the index sets: $\Psi_{k + 1, \ell} \leftarrow \Psi_{k , \ell} \cup \{k\}$ and $\Psi_{k + 1, \ell'}\leftarrow \Psi_{k , \ell'}$ for $\ell' \in [L]\backslash \{\ell\}$. \label{alg1:line:19}
        \ENDIF
        \STATE $\ell \leftarrow \ell + 1$. 
        \ENDWHILE
         \STATE For $\ell\in[L]$ such that $\Psi_{k+1,\ell}\neq \Psi_{k,\ell}$, update the estimators as follows: \label{alg1:line:17} \begin{align} \hat\bSigma_{k + 1, \ell} \leftarrow \hat\bSigma_{k, \ell} + w_k^2 \ab_k \ab_k^{\top}, \hat{\bbb}_{k + 1, \ell} \leftarrow \hat\bbb_{k, \ell} + w_k^2 \cdot r_k \ab_k, \hat\btheta_{k + 1, \ell} \leftarrow \hat\bSigma_{k + 1, \ell}^{-1} \hat{\bbb}_{k + 1, \ell}. \notag \end{align}
        Compute the adaptive confidence radius $\hat\beta_{k+1, l}$for the next round according to \eqref{eq:def:beta}. \label{alg1:line:18}
         \STATE For $\ell\in[L]$ such that $\Psi_{k+1,\ell}=\Psi_{k,\ell}$, let $\hat\bSigma_{k + 1, \ell} \leftarrow \hat\bSigma_{k, \ell}, \hat\bbb_{k + 1, \ell} \leftarrow \hat\bbb_{k, \ell}, \hat\btheta_{k + 1, \ell} \leftarrow \hat\btheta_{k, \ell}, \hat\beta_{k + 1, \ell} \leftarrow \hat\beta_{k, \ell}.$
        \ENDFOR
        \end{algorithmic}
    \end{algorithm*}
    
    \subsubsection{SupLin with Adaptive Variance-Aware Exploration}
    
    As discussed in the last subsection, in order to exploit Theorem \ref{thm:bernstein1} effectively in the heteroscedastic linear bandits setting, we need the uncertainty/bonus term $\|\ab_k\|_{\hat\bSigma_{k-1}^{-1}}$ to be small, where $\hat\bSigma_{k-1}$ is the covariance matrix of $\ab_k$. However, such a term is in the order of $O(1)$ in the worst case. 
    Our algorithm 
    partition the observed contexts into different layers such that the uncertainty of the contexts within each layer is small. Our algorithm is displayed in Algorithm \ref{alg:1}, namely  \texttt{SAVE}. 
    \newline
    
    \noindent\textbf{Overall algorithm structure} In general, Algorithm \ref{alg:1} shares a similar multi-layer structure as \texttt{SupLinUCB} in \citet{chu2011contextual}. Algorithm \ref{alg:1} maintains $L$ context sets $\Psi_{k,\ell}, \ell \in [L]$ 
    at the $k$-th round. The goal of Algorithm \ref{alg:1} at the $k$-th round is to select $\ab_k$ which maximizes $\la \ab, \btheta^*\ra$. Since $\btheta^*$ is unknown, the selection process is based on $L$  number of estimates of $\btheta^*$, which we denote them by $\hat\btheta_{k,\ell}, \ell\in[L]$. $\hat\btheta_{k,\ell}$ is the solution to some regression problem over contexts in $\Psi_{k,\ell}$ and their corresponding rewards. Starting from $\ell=1$, the decision set $\cA_{k,\ell}$ will keep `shrinking' by eliminating all $\ab \in \cA_{k,\ell}$ which are unlikely to be the maximizer of $\la \ab, \btheta^*\ra$ (notably, since $\btheta^*$ is unknown, here $\btheta^*$ needs to be replaced by $\hat\btheta_{k,\ell}$, as displayed in Line \ref{alg1:line:12}). 
    The elimination process will not stop until some action $\ab$ with large uncertainty $\|\ab\|_{\hat\bSigma_{k,\ell}^{-1}}$ emerges. Then Algorithm \ref{alg:1} will either select the action $\ab$ with a large uncertainty (Line \ref{alg1:line:14} to Line \ref{alg1:line:19}), or the action which maximizes the upper confidence bound of estimated reward if there is no action with a large uncertainty (Line \ref{line:selection} to Line~\ref{alg1:line:10}). 
    The context set $\Psi_{k,\ell}$ will be updated by adding $k$ into it only when $\ab_k$ enjoys a large uncertainty.
    \newline
    
    \noindent\textbf{Construction of the estimate $\hat\btheta_{k,\ell}$} The first difference between our algorithm and \texttt{SupLinUCB} is the construction of the estimate $\hat\btheta_{k,\ell}$. Unlike the \emph{unweighted} ridge regression estimator applied in \texttt{SupLinUCB}, we employ a \emph{weighted ridge-regression} estimator as follows \begin{align*} 
        \forall k \in [K]\ \text{and}\ \ell\in[L], \quad \hat \btheta_{k, \ell} = \argmin_{\btheta \in \RR^d} \sum_{i \in \Psi_{k, \ell}} w_i^2\big(r_i - \la \btheta, \ab_i\ra\big)^2 + 2^{-2\ell} \lambda \|\btheta\|_2^2, 
    \end{align*}
    where the weight $w_i$ is chosen such that 
    for $\forall \ell \in [L]\ \text{and}\ i \in \Psi_{k, \ell}, \ \|w_i \ab_i\|_{\hat \bSigma_{i, \ell}^{-1}} = 2^{-\ell}. $
    We explain here why we want to adopt such a weighted regression scheme. In particular, the estimate $\hat \btheta_{k,\ell}$ can be regarded as $\bmu_k$ in Theorem \ref{thm:bernstein1}. By our construction of $w_i$, we can ensure that the context $\xb_i$ in Theorem \ref{thm:bernstein1}, which is $w_i\ab_i$ here, enjoys a uniform upper bound on the uncertainty, i.e., $\|\xb_i\|_{\Zb_i^{-1}} \leq 2^{-\ell}$. Such a result can further imply that $\|\hat\btheta_{k,\ell} - \btheta^*\|_{\hat\bSigma_{k,\ell}}$ is in the order of $O(2^{-\ell})$, which is tighter than the vanilla bound $O(1)$ deduced by previous works.



    \begin{remark} 
    It is worth noting that weighted ridge-regression technique has been used in heteroscedastic bandit setting \citep{kirschner2018information,zhou2021nearly,zhou2022computationally} for the known variance case. 
    The most related work to ours is \citet{zhou2021nearly}, which applies the following weighted ridge-regression estimator 
        $\btheta_{k} = \argmin_{\btheta \in \RR^d} \sum_{i = 1}^{k - 1} \frac{1}{\sigma_i^2}\big(r_i - \la \btheta, \ab_i\ra\big)^2 + \lambda \|\btheta\|_2^2,$
    where the $1 / \sigma_i^2$ weight is introduced to normalize the variance of noise. Our weight $w_i$, in contrast, is set to reweight the feature vectors such that they have the same elliptical norm $\|\ab_i\|_{\hat\bSigma_{i, \ell}^{-1}}$.  
    Weighted ridge-regression technique has also been applied to other bandit settings such as linear multi-resource allocation \citep{lattimore2015linear} and corruption-robust linear bandits \citep{he2022nearly1}. In particular, \citet{he2022nearly1} adopts a similar weight $w_i = O\big(1 / \|\ab_i\|_{\hat\bSigma_i^{-1}}^{1 / 2}\big)$ to balance the effect of adversarial corruption and stochastic noise. Nevertheless, the specific bandit problems they are solving are quite different from ours. 
    \end{remark}
    \noindent\textbf{Adaptive variance-aware exploration}
    According to previous discussion, we can bound the estimation error of $\hat\btheta_{k,\ell}$ following Theorem \ref{thm:bernstein1}, 
    which leads to a confidence bound of $\btheta^*$, i.e., $\{\btheta: \|\btheta - \hat\btheta_{k,\ell}\|_{\hat\bSigma_{k,\ell}} \le \tilde O(2^{-\ell}\cdot \sqrt{\sum_{i \in \Psi_{k, \ell}} w_i^2\sigma_i^2} + 2^{-\ell}R)\}$. 
    This can be used in the arm selection step (Line \ref{line:selection} and Line \ref{alg1:line:12}). 
    However, such a confidence set requires the knowledge of variances $\sigma_i^2$ apriori. To address this issue, we need to replace $\sigma_i^2$ with their \emph{empirical estimator}. In detail, since $\sigma_i^2 = \EE[(r_i - \la \btheta^*, \ab_i\ra)^2|\ab_{1:i-1}, \epsilon_{1:i-1}]$, 
    we simply use an \emph{one-point plug-in estimator} $(r_i - \la \hat\btheta_{k,\ell}, \ab_i\ra)^2$. 
    With such an estimator, we define the confidence radius $\beta_{k+1,\ell}$ at round $k+1$ and layer 
    $\ell$ as
    \begin{align} 
        \beta_{k + 1, \ell} &:= 16 \cdot 2^{-\ell} \sqrt{\left(8\hat\Var_{k + 1, \ell}  + 6R^2 \log(4(k + 1)^2 L / \delta) + 2^{-2\ell + 4}\right)\log(4k^2 L / \delta)} \notag\\&\quad + 6 \cdot 2^{-\ell} R \log(4k^2 L / \delta) + 2^{-\ell + 1} \label{eq:def:beta}, 
    \end{align}
    where 
    \begin{align*} \hat\Var_{k + 1, \ell} := \begin{cases}\sum_{i \in \Psi_{k + 1, \ell}} w_i^2 \big(r_i - \la \hat\btheta_{k + 1, \ell}, \ab_i \ra \big)^2, & 2^\ell \ge 64 \sqrt{\log\left(4(k + 1)^2 L /\delta\right)} \\
    R^2 \left|\Psi_{k + 1, \ell}\right|,  & \text{otherwise}.
    \end{cases} \end{align*}
    
    We would like to emphasize that although our used one-point estimator $(r_i - \la \hat\btheta_{k,l}, \ab_i\ra)^2$ might be an inaccurate estimator of the target $\sigma_i^2$ for some round $i$, the \emph{weighted summation} of the one-point estimators $\sum w_i^2(r_i - \la \btheta_{k,l}, \ab_i\ra)^2$ actually serves as a sufficiently accurate estimator of the total variance $\sum w_i^2\sigma_i^2$. That is because our employed weight can effectively `calibrate' the term $(r_i - \la \btheta_{k,l}, \ab_i\ra)^2$ and reduce its error, leading to an accurate estimate when these terms are summed together.

    
    \subsubsection{Computational Complexity} \label{sec:cc}
    
    At each round $k \in [K]$, the learner executes the arm elimination step (Line \ref{alg1:line:12} in Algorithm \ref{alg:1}) 
    for $O(L)$ times, and then applies Sherman-Morrison formula \citep{golub2013matrix} and matrix multiplication to update the estimator in $O(d^2)$ time (Line \ref{alg1:line:17}). 
    Note that we need to compute the confidence radius at each round in Line \ref{alg1:line:18}, which will take $O(k)$ time if we compute it directly. However, we can compute $\hat\Var_{k + 1, \ell}$ by 
    $\hat\Var_{k + 1, \ell} = \sum_{i \in \Psi_{k + 1, \ell}} w_i^2r_i^2 - 2 \hat\btheta_{k + 1, \ell}^\top \cdot  \sum_{i \in \Psi_{k + 1, \ell}} w_i r_i \ab_i + \hat\btheta_{k + 1, \ell}^\top \bigg(\sum_{i \in \Psi_{k + 1, \ell}} w_i^2 \ab_i \ab_i^\top \bigg)\hat\btheta_{k + 1, \ell}, $
    where the first term can be computed in $O(1)$ time at each round, the second term can be computed in $O(d)$ time by maintaining the prefix sum of $w_ir_i \cdot \ab_i$ and the third term can be computed in $O(d^2)$ time by maintaining the value of the weighted covariance matrix. By adding these steps together, we can conclude that the time complexity of Algorithm \ref{alg:1} is $O(K|\cD|Ld^2)$. 
    
    \subsection{Regret Bounds}
    We provide the regret guarantee of Algorithm \ref{alg:1} in the following theorem.
    
    \begin{theorem} \label{thm:regret1}
        Suppose that for all $k \ge 1$ and all $\ab \in \cD_k, \|\ab\|_2 \le A, \|\btheta^*\|_2 \le 1, \ \la \ab, \btheta^* \ra \in [-1, 1]$. If $\{\beta_{k, \ell}\}_{k\ge 1, \ell \in [L]}$ is defined in \eqref{eq:def:beta} and $\alpha = 1 / (R \cdot K^{3 / 2})$, then the cumulative regret of Algorithm~\ref{alg:1} is bounded as follows with probability at least $1 - 3\delta$: 
        \begin{align*} 
            \regret(K) = \tilde{O}\bigg(d \sqrt{\sum_{k = 1}^K \sigma_k^2} + d R  + d\bigg). 
        \end{align*}
    \end{theorem}
    \begin{remark}
    If we treat $R$ as a constant, the  regret can be simplified as  $ \tilde{O}\Big(d \sqrt{\sum_{k = 1}^K \sigma_k^2}   + d\Big)$. Compared with Weighted OFUL+ \citep{zhou2022computationally}, our algorithm achieves the same order of regret guarantee and does not require any prior knowledge about the variance $\sigma_k$. Compared with VOFUL2 \citep{kim2021improved}, our \texttt{SAVE} algorithm improves the regret from $\tilde{O}\Big(d^{1.5} \sqrt{\sum_{k = 1}^K \sigma_k^2}   + d^2\Big)$ to $ \tilde{O}\Big(d \sqrt{\sum_{k = 1}^K \sigma_k^2}   + d\Big)$. Furthermore, \texttt{VOFUL2} needs to perform the arm elimination for each possible direction $\bmu$ in the $d$-dimension unit ball, which requires an exponential computational time (See the discussion in Section \ref{subsec:challenges}).
    \end{remark}
    \begin{remark}
        Consider the deterministic reward setting where $\sigma_k=0$ holds for all round $k\in[K]$. If we treat $R$ as a constant, then Theorem \ref{thm:regret1} suggests an $\tilde O(d)$ regret guarantee, which matches the $\Omega(d)$ lower bound up to logarithmic factors \citep{chu2011contextual}. 
    \end{remark}

    \section{Variance-Aware Learning for Linear Mixture MDPs}
    \label{sec-4}
    In this section, we apply the techniques developed in Section \ref{sec:bandits} to reinforcement learning, and propose a variance-aware algorithm for linear mixture MDPs. 
    \begin{algorithm}[!ht] 
        \caption{UCRL-AVE} \label{alg:2}
        \begin{algorithmic}[1]
        \REQUIRE Regularization parameter $\lambda > 0$, $\alpha > 0$, $B$, an upper bound on the $\ell_2$-norm of $\btheta^*$. 
        \STATE Set $L = \lceil \log_2(1 / \alpha) \rceil$. 
        \STATE Initialize: $\hat\bSigma_{0, H + 1, \ell} \leftarrow 2^{-2\ell} \lambda \cdot \Ib$, $\hat\bbb_{0, H + 1, \ell} \leftarrow \zero$, $\hat\btheta_{1, \ell} \leftarrow \zero$ for all $\ell \in [L]$. 
        \FOR{$k=1,\ldots, K$}
            \STATE $V_{k, H + 1}(\cdot) \leftarrow 0$. 
            \STATE Update the current estimators: $\hat\bSigma_{k, 1, \ell} \leftarrow \hat\bSigma_{k - 1, H + 1, \ell}$, $\hat\bbb_{k, 1, \ell} \leftarrow \hat\bbb_{k - 1, H + 1, \ell}$, $\hat \btheta_{k, \ell} \leftarrow \hat\bSigma_{k, 1, \ell}^{-1} \hat\bbb_{k, 1, \ell}$ for all $\ell \in [L]$. 
            \STATE Compute $\hat\beta_{k, \ell}$ according to \eqref{eq:def:beta:mdp}.  
            \FOR{$h = H, \ldots, 1$}
            \STATE \begin{small}$Q_{k, h}(\cdot, \cdot) \leftarrow \min\Big\{1, \min_{\ell \in [L]}\Big[r(\cdot, \cdot) + \big\la\hat \btheta_{k, \ell}, \bphi_{V_{k, h + 1}}(\cdot, \cdot)\big\ra + \hat \beta_{k, \ell} \left\|\bphi_{V_{k, h + 1}}(\cdot, \cdot)\right\|_{\hat \bSigma_{k, 1, \ell}^{-1}}\Big]\Big\}$. \end{small} \label{alg2:line:7}
            \STATE $\pi_k(\cdot, h) \leftarrow \argmax_{a \in \cA} Q_{k, h}(\cdot, a)$,\quad  $V_{k, h}(\cdot) \leftarrow \max_{a \in \cA} Q_{k, h}(\cdot, a)$. 
            \ENDFOR
        \STATE Observe $s_1^k$. 
        \FOR{$h = 1, \ldots H$}
            \STATE Take action $a_h^k \leftarrow \pi_k(s_h^k, h)$ and observe $s_{h + 1}^k$. \label{alg2:line:13}
            \STATE $\cL_{k, h} \leftarrow \Big\{\ell \in [L] \Big| \left\|\bphi_{V_{k, h + 1}}(s_h^k, a_h^k)\right\|_{\bSigma_{k, h, \ell}^{-1}} \ge 2^{-\ell}\Big\}$. \label{alg2:line:14}
            \STATE Set $\ell_{k, h} \leftarrow \begin{cases}
                L + 1, & \cL_{k, h}  = \varnothing \\
                \min\left(\cL_{k, h}\right), & \text{otherwise}
            \end{cases} $. \label{alg2:line:15}
            \IF{$\ell_{k, h} \neq L + 1$}
                \STATE $w_{k, h} \leftarrow 2^{-\ell_{k, h}}/{\left\|\bphi_{V_{k, h + 1}}(s_h^k, a_h^k)\right\|_{\bSigma_{k, h, \ell_{k, h}}^{-1}}}$. \label{alg2:line:17}
                \STATE $\hat \bSigma_{k, h + 1, \ell_{k, h}} \leftarrow \hat \bSigma_{k, h, \ell_{k, h}} + w_{k, h}^2 \bphi_{V_{k, h + 1}}(s_h^k, a_h^k)\bphi_{V_{k, h + 1}}(s_h^k, a_h^k)^\top$. 
                \STATE $\hat \bbb_{k, h + 1, \ell_{k, h}} \leftarrow \hat \bbb_{k, h, \ell_{k, h}} + w_{k, h}^2 V_{k, h + 1}(s_{h + 1}^k)\bphi_{V_{k, h + 1}}(s_h^k, a_h^k)$. 
            \ENDIF
            \STATE $\hat \bSigma_{k, h + 1, \ell} \leftarrow \hat \bSigma_{k, h, \ell}$, $\hat \bbb_{k, h + 1, \ell} \leftarrow \hat \bbb_{k, h, \ell}$ for all $\ell \in [L]$ and $\ell \neq \ell_{k, h}$. 
        \ENDFOR
        
        \ENDFOR
        \end{algorithmic}
    \end{algorithm}
    
    \subsection{Problem Setup}
    
    \noindent\textbf{Episodic MDPs.} A time-homogenous episodic MDP \citep{puterman2014Markov} is denoted by a tuple $M = M(\cS, \cA, H, r, \PP)$. Here, $\cS$ is the state space, $\cA$ is a finite action space,  $H$ is the planning horizon (i.e., length of each episode), $r: \cS \times \cA \rightarrow [0,1]$ is a deterministic reward function, $\PP(s'|s,a)$ is the transition probability function denoting the probability of transition from state $s$ to state $s'$ under action $a$. A policy $\pi: \cS \times [H] \rightarrow \cA$ is a function which maps a state $s$ and the stage number $h$ to an action $a$. 
    For any policy $\pi$ and stage $h\in [H]$, we define the following action-value function $\qvalue_h^{\pi}(s,a)$ and value function $\vvalue_h^{\pi}(s)$ as follows
    \begin{align}
    \qvalue^{\pi}_h(s,a) =r(s,a) + \EE\bigg[\sum_{h'=h+1}^H r\big(s_{h'}, \pi(s_{h'},h')\big)\bigg| s_h=s,a_h=a\bigg], \quad
    \vvalue_h^{\pi}(s) = \qvalue_h^{\pi}(s, \pi(s,h)),\notag
    \end{align}
    where $s_{h'+1}\sim \PP(\cdot|s_{h'},a_{h'})$.
    We further define the optimal value function $V_h^*$ and the optimal action-value function $\qvalue_h^*$ as $V_h^*(s) = \max_{\pi}\vvalue_h^{\pi}(s)$ and $\qvalue_h^*(s,a) = \max_{\pi}\qvalue_h^{\pi}(s,a)$. In addition, for any function $\vvalue: \cS \rightarrow \RR$, we denote $[\PP \vvalue](s,a)=\EE_{s' \sim \PP(\cdot|s,a)}\vvalue(s')$. Therefore, for each stage $h\in[H]$ and policy $\pi$, we have the following Bellman equation, as well as the Bellman optimality equation:
    \begin{align}
        \qvalue_h^{\pi}(s,a) = \reward(s,a) + [\PP\vvalue_{h+1}^{\pi}](s,a), 
    \quad \qvalue^{*}(s,a) = \reward(s,a) + [\PP\vvalue_{h+1}^{*}](s,a), \notag
    \end{align}
    where $\vvalue^{\pi}_{H+1}(\cdot)=\vvalue^{*}_{H+1}(\cdot)=0$. At the beginning of episode $k$, the agent chooses a policy $\pi$ to guide its actions throughout the episode. At each stage $h\in[H]$, the agent observes the state $s_h^k$, chooses an action by the policy $\pi$ and observes the next state with $s_{h+1}^k \sim \PP(\cdot|s_h^k,a_h^k)$. 
    
    Following previous work on horizon-free regret in linear mixture MDPs \citep{zhang2021variance, kim2021improved, zhou2022computationally}, we consider the setting where the total reward (i.e., return of an episode) is bounded by $1$. 
    \begin{assumption} 
        For any policy $\pi$, let $(s_h, a_h)_{h=1}^H$ be one trajectory following $\pi$, then $$\sum_{h \in [H]} r(s_h, a_h) \le 1$$ almost surely. 
    \end{assumption}
    For simplicity, let $\left[\VV V\right](s, a) = \left[\PP V^2\right](s, a) - \left(\left[\PP V\right](s, a)\right)^2$ denote the conditional variance of $V$ conditioned on $(s,a)$. We define the following instance-dependent quantity: 
    \begin{align} 
    \Var_K^* = \sum_{k = 1}^K \sum_{h = 1}^H [\VV V_{h + 1}^*](s_h^k, a_h^k). \label{def:newv}
    \end{align}
    The quantity \eqref{def:newv} characterizes the stochasticity of the MDP under the optimal policy. For a deterministic MDP where the transition function is deterministic , we have $\Var_K^* = 0$. Similar quantities have been considered in \citet{maillard2014hard, zanette2019tighter}, and the same quantity has been proposed by a concurrent work \citep{zhou2023sharp} on tabular RL.


    \noindent\textbf{Linear Mixture MDPs.}
    We consider a special MDP class called \emph{linear mixture MDPs}. 
    \begin{definition}[Episodic linear mixture MDPs, \citealt{jia2020model, ayoub2020model}]\label{assumption-linear}
        An episodic MDP $\cM(\cS, \cA, H, \reward, \PP)$ is a homogeneous, episodic $B$-bounded linear mixture MDP if there exists vectors $\btheta^* \in \RR^d$ with $\|\btheta^*\|_2 \le B$ and $\bphi(\cdot | \cdot, \cdot)$ satisfying \eqref{eq:feature-bound}, such that for each $(s,a) \in \cS \times \cA$, $s' \in \cS$ and stage $h \in [H]$, $\PP(s'|s,a) = \big\la \bphi(s'| s,a), \btheta^*\big\ra$. Moreover, $\bphi$ satisfies that for any bounded function $V: \cS \to [0, 1]$ and any tuple $(s, a) \in \cS \times \cA$, 
    \begin{align} 
        \|\bphi_V(s, a)\|_2 \le 1, \text{where}\ \bphi_V(s, a) := \sum_{s' \in \cS} \bphi(s'|s, a) V(s'). \label{eq:feature-bound}
    \end{align}
    \end{definition}
    The goal of the agent is to minimize the following cumulative regret at the first $K$ rounds: \begin{align} 
        \regret(K) = \sum_{k \in [K]}\big[V_1^*(s_1^k) - V_1^{\pi^k}(s_1^k)\big]. \notag
    \end{align}
    \subsection{The Proposed Algorithm}
    
    We present an adaptive variance-aware algorithm named UCRL with \textbf{A}daptive \textbf{V}ariance-Aware \textbf{E}xploration (\texttt{UCRL-AVE}) in Algorithm \ref{alg:2}. The backbone of our algorithm is the \emph{value-targeted-regression} scheme proposed by \texttt{UCRL-VTR} \citep{jia2020model, ayoub2020model}. In detail, Algorithm~\ref{alg:2} aims to estimate the optimal value function $Q^*_h$ by $Q_{k,h}$, utilizing the Bellman optimal equation. Since $\PP V_{k,h+1}$ is not attractable ($\PP$ is unknown), Algorithm \ref{alg:2} uses the fact that $\PP V_{k,h+1}(s,a) = \la \bphi_{V_{k,h+1}}(s,a), \btheta^*\ra$ is a linear function of the feature $\bphi_{V_{k,h+1}}(s,a)$, and estimates $\PP V_{k,h+1}(s,a)$ by a plug-in estimator $\la \bphi_{V_{k,h+1}}(s,a), \hat\btheta_{k}\ra$, where $\hat\btheta_{k}$ is the estimate of $\btheta^*$. Then \texttt{UCRL-VTR} computes $Q_{k,h}$ by the \emph{upper confidence bound} of the empirical estimator with truncation (Line \ref{alg2:line:7}).
    
    The main difference between Algorithm \ref{alg:2} and \texttt{UCRL-VTR} is the construction of $\hat\btheta_k$: instead of using a single estimate, Algorithm \ref{alg:2} maintains $L$ estimates $\hat\btheta_{k,\ell}$, constructed on a multi-layer structure of feature vectors. We highlight several important technical innovations here.


    \paragraph*{Multi-layer structure of feature vectors.} We first demonstrate how Algorithm \ref{alg:2} utilizes $L$ number of estimates $\hat\btheta_{k,\ell}$ to build the value function estimate $Q_{k,h}$, then we show how Algorithm \ref{alg:2} updates $\hat\btheta_{k,\ell}$ accordingly. Algorithm \ref{alg:2} constructs $Q_{k,h}$ as the minimum of $L$  optimistic estimates computed by $\la \hat\btheta_{k,\ell}, \bphi_{V_{k,h+1}}\ra$. The minimum step makes the estimate $Q_{k,h}$ tighter than that in \texttt{UCRL-VTR}.

    Similar to Algorithm \ref{alg:1}, $\hat\btheta_{k,l}$ is the solution to some regression problem over the features $\phi_{V}(s,a)$ and their corresponding target values. For simplicity, we define the following subsets of $[K] \times [H]$: \begin{align} 
        \Psi_{k, \ell} = \left\{(i, h) \in [k - 1] \times [H] | \ell_{i, h} = \ell\right\}, \quad \ \text{for} \ k \in [K + 1], \ell \in [L + 1], \label{eq:def:index} 
    \end{align}which represents the indices of feature vectors in layer $\ell$ at the beginning of round $k$. Note that $\hat\btheta_{k,\ell}$ will be updated if the feature $\phi_{V_{h + 1}^k}(s_h^k, a_h^k)$ is added to the feature set $\Psi_{k,\ell}$. The rule that whether to add such a feature or not is based on the uncertainty of $\phi_{V_{h + 1}^k}(s_h^k, a_h^k)$ within the feature set $\Psi_{k,\ell}$, which is similar to the multi-layer structure adopted by \citet{he2021uniform} for uniform-PAC bounds in linear MDPs. Finally, $\hat\btheta_{k,\ell}$ is computed as the solution to the weighted regression problem over the feature set $\Psi_{k,\ell}$, where the weight $w_{k,h}$ is selected to guarantee that $\left\|w_{k, h} \bphi_{V_{k, h + 1}}(s_h^k, a_h^k)\right\|_{\bSigma_{k, h, \ell_{k, h}}^{-1}} = 2^{-\ell_{k, h}}$, similar to that in Algorithm \ref{alg:1}.




    
    \paragraph*{Adaptive variance-aware exploration. } Similar to Algorithm \ref{alg:1}, we will also face the problem to construct a confidence set of $\btheta^*$ \emph{without} knowing the variance of value functions $V_{k,h+1}$. Here we take the same approach: to replace the variance of $V_{k, h+1}$, $\PP[V_{k,h+1} - \PP V_{k,h+1}]^2$ with its one-point empirical estimate $(V_{i,h+1}(s_{h+1}^i) - \la \hat\btheta_{k,l}, \bphi_{V_{i,h+1}}(s_h^i, a_h^i)\ra)^2$. In detail, the confidence radius is: 
    \begin{small}
    \begin{align} 
    \hat \beta_{k, \ell} &:= 16 \cdot 2^{-\ell} \sqrt{\big(8\hat\Var_{k, \ell} + 8 \log(4k^2H^2 L /\delta) + 2^{-2\ell + 5} \cdot \lambda B^2 \big)\log(4k^2H^2 L / \delta)}\notag \\&\quad + 6 \cdot 2^{-\ell} \log(4k^2H^2 L /\delta) + 2^{-\ell} \sqrt{\lambda} \cdot B, \label{eq:def:beta:mdp}
    \end{align}\end{small}
    where 
    \begin{small}
    \begin{align*}\hat\Var_{k, \ell} = \begin{cases} 8 \sum_{(i, h) \in \Psi_{k, \ell}} w_{i, h}^2 \big(V_{i, h + 1}(s_{h + 1}^i) - \la \hat \btheta_{k, \ell}, \bphi_{V_{i, h + 1}}(s_h^i, a_h^i)\ra\big)^2, & 2^\ell \ge 64\sqrt{\log(4k^2H^2 L / \delta)}, \\
    |\Psi_{k, \ell}|, & \text{otherwise}.
    \end{cases}\end{align*}\end{small}
    
    We can adopt the method discussed in Subsection \ref{sec:cc} to compute $\hat\Var_{k, \ell}$ in an efficient way. We call the construction of the confidence set along with its radius $\hat\beta_{k,\ell}$ as \emph{adaptive variance-aware exploration}. 
    
    Compared with \texttt{UCRL-VTR+} \citep{zhou2021nearly} and \texttt{HF-UCRL-VTR+} \citep{zhou2022computationally}, our algorithm does not need to estimate the conditional variance using another ridge regression estimator on the second-order moment of value functions. Furthermore, in contrast to \texttt{HF-UCRL-VTR+} \citep{zhou2022computationally}, our algorithm does not explicitly estimate the high-order moments of value functions. Thus, our algorithm is much simpler. It is also worth noting that the multi-layer structure in Algorithm \ref{alg:2} is an alternative of the \texttt{SupLinUCB}-type design in Algorithm \ref{alg:1}. Since linear bandits can be seen as a special case of linear mixture MDPs, Algorithm \ref{alg:2} implies another algorithm for heteroscedastic linear bandits, which enjoys the same regret guarantee as Algorithm \ref{alg:1}.   
    
    \subsection{Regret Bounds}
    We provide the regret guarantee of Algorithm \ref{alg:2} in the following theorem.
    \begin{theorem} \label{thm:regret2}
    Set $\hat \beta_{k, \ell}$ as in \eqref{eq:def:beta:mdp}, $\alpha = 1 / (KH)^{3 / 2}$ and $\lambda = 1 / B^2$ in Algorithm \ref{alg:2}. 
    Then with probability at least $1 - (4\lceil \log_2 2HK \rceil + 9)\delta$ , 
    the regret of Algorithm \ref{alg:2} is bounded by: 
        \begin{align*} 
            \regret(K) = \tilde{O}\Big(d \sqrt{\Var_K^*} + d^{2}\Big). 
        \end{align*}
    \end{theorem}
    
    \begin{corollary}\label{coro:regret2}
    Under the same conditions as Theorem \ref{thm:regret2}, with probability at least $1 - (4\lceil \log_2 2HK \rceil + 10)\delta$, the regret of Algorithm \ref{alg:2} is bounded by: 
    \begin{align*}
        \regret(K) = \tilde{O}\big(d\sqrt{K} + d^2\big). 
    \end{align*}
    \end{corollary}
    
    \begin{remark}
        Our regret given by Theorem \ref{thm:regret2} is variance-dependent, which means that the regret of \texttt{UCRL-AVE} is smaller when the conditional variance of the optimistic value function is smaller. In the deterministic case where all the transitions in the MDP is deterministic, our regret reduces to $\tilde{O}(d^2)$, with only a logarithmic dependence on $K$. Additionally, the regret in Corollary \ref{coro:regret2} matches the regret of \texttt{HF-UCRL-VTR+} proposed by \citet{zhou2022computationally}, which is the worst-case regret and matches the minimax lower bound \citep{zhou2022computationally}.
    \end{remark}
    
    \section{Conclusion and Future Work}
    
    In this paper, we consider variance-aware learning in linear bandits and linear mixture MDPs. We propose a computationally efficient algorithm \texttt{SAVE} for heteroscedastic linear bandits, which achieves a variance-dependent regret, matching the minimax regret bounds in both the worst case and the deterministic reward case. For linear mixture MDPs, we further extend our techniques and propose an algorithm dubbed \texttt{UCRL-AVE}, attaining a tighter problem-dependent horizon-free regret bound. We leave for future work the generalization of our work to RL with nonlinear function approximation. 
    
    

    \appendix

    \section{Proof of Theorem \ref{thm:bernstein1}}
    
    \begin{proof}
        For simplicity, we introduce the following definitions: 
            \begin{align} 
                \db_0 = 0,  \db_k = \sum_{i = 1}^k \xb_i \eta_i, q_0 = 0, q_k = \|\db_k\|_{\Zb_k^{-1}}, \cI_k = \mathds{1}\{0 \le s \le k, q_s \le \beta_s\}, \notag
            \end{align}
            where $k \ge 1$ and we further define $\beta_0 = 0$, $\cE_0 = 1$. According to these definitions, the term $q_k$ can be upper bounded by the following decomposition: 
            \begin{align} 
                q_k^2 &= (\db_{k - 1} + \xb_k \eta_k)^\top \Zb_k^{-1}(\db_{k - 1} + \xb_k \eta_k) \notag
                \\&= \db_{k - 1}^\top \Zb_k^{-1} \db_{k - 1} + \underbrace{2\eta_k \xb_k^\top \Zb_k^{-1} \db_{k - 1}}_{I_{1, k}} + \underbrace{\eta_k^2 \xb_k^\top \Zb_k^{-1} \xb_k}_{I_{2, k}} \notag
                \\&\le q_{k - 1}^2 + I_{1, k} + I_{2, k}, \label{eq:conf:recur}
            \end{align}
            where the inequality holds since $\Zb_k=\Zb_{k - 1} +\xb_{k}\xb_{k}^{\top}  \succeq \Zb_{k - 1}$. For the term $I_{1,k}$, from the matrix inversion lemma, we have the following equation:  \begin{align} 
                I_{1, k} &= 2\eta_k \bigg(\xb_k^\top \Zb_{k - 1}^{-1} \db_{k - 1} - \frac{\xb_k \Zb_{k - 1}^{-1} \xb_k \xb_k^\top \Zb_{k - 1}^{-1} \db_{k - 1}}{1 + \|\xb_{k}\|_{\Zb_{k - 1}^{-1}}^2}\bigg) \notag
                \\&= 2\eta_k \bigg(\xb_k^\top \Zb_{k - 1}^{-1} \db_{k - 1} - \frac{\|\xb_{k}\|_{\Zb_{k - 1}^{-1}}^2 \xb_k^\top \Zb_{k - 1}^{-1} \db_{k - 1}}{1 + \|\xb_{k}\|_{\Zb_{k - 1}^{-1}}^2}\bigg) \notag
                \\&= 2\eta_k \cdot \frac{\xb_k^\top \Zb_{k - 1}^{-1} \db_{k - 1}}{1 + \|\xb_{k}\|_{\Zb_{k - 1}^{-1}}^2}. \notag
            \end{align}
    Taking a summation over the term $I_{1,k}$ with respect to the indicator function $\cE_{k-1}$, we have the following equation:       
    \begin{align} 
                \sum_{i = 1}^k I_{1, i} \cdot \cI_{i - 1} &= 2\sum_{i = 1}^k \eta_i \cdot \frac{\xb_i^\top \Zb_{i - 1}^{-1} \db_{i - 1}}{1 + \|\xb_{i}\|_{\Zb_{i - 1}^{-1}}^2} \cI_{i - 1}. \label{eq:I1}
            \end{align}
    Now, we can derive an upper bound for this summation by Freedman's inequality. In detail, for each round $i\in[k]$, we have
            \begin{align} 
                \Bigg|\eta_i \cdot \frac{\xb_i^\top \Zb_{i - 1}^{-1} \db_{i - 1}}{1 + \|\xb_{i}\|_{\Zb_{i - 1}^{-1}}^2} \cI_{i - 1}\Bigg| \le R \Bigg|\frac{\|\xb_i\|_{\Zb_{i - 1}^{-1}} \|\db_{i - 1}\|_{\Zb_{i - 1}^{-1}}}{1 + \|\xb_{i}\|_{\Zb_{i - 1}^{-1}}^2} \Bigg| \cI_{i - 1}
                \le R\Bigg|\frac{\|\xb_i\|_{\Zb_{i - 1}^{-1}} \beta_{i - 1}}{1 + \|\xb_{i}\|_{\Zb_{i - 1}^{-1}}^2}\Bigg| \le R \beta_{k} \rho, \notag
            \end{align}
            where the first inequality holds due to Cauchy-Schwarz inequality, the second inequality holds due to the definition of indicator function $\cE_{i-1}$ and the last inequality holds due to $\rho \ge \|\xb_i\|_{\Zb_{i-1}^{-1}} $. In addition, for each round $i\in[k]$, we have
            \begin{align}
                \EE\Bigg[\eta_i \cdot \frac{\xb_i^\top \Zb_{i - 1}^{-1} \db_{i - 1}}{1 + \|\xb_{i}\|_{\Zb_{i - 1}^{-1}}^2} \cI_{i - 1}\bigg|\cG_k\Bigg] = 0,\notag
            \end{align}
            and the summation of variance is upper bounded by
            \begin{align} 
                \sum_{i = 1}^k \EE\Bigg[\bigg(\eta_i \cdot \frac{\xb_i^\top \Zb_{i - 1}^{-1} \db_{i - 1}}{1 + \|\xb_{i}\|_{\Zb_{i - 1}^{-1}}^2} \cI_{i - 1}\bigg)^2 \Bigg| \cG_i\Bigg]
                &= \sum_{i = 1}^k \Bigg( \frac{\xb_i^\top \Zb_{i - 1}^{-1} \db_{i - 1}}{1 + \|\xb_{i}\|_{\Zb_{i - 1}^{-1}}^2} \cI_{i - 1}\Bigg)^2\EE[\eta_i^2 | \cG_i] \notag
                \\&\le \sum_{i = 1}^k \Bigg( \frac{\|\xb_{i}\|_{\Zb_{i - 1}^{-1}} \|\db_{i - 1}\|_{\Zb_{i - 1}^{-1}} \cI_{i - 1}}{1 + \|\xb_{i}\|_{\Zb_{i - 1}^{-1}}^2} \Bigg)^2\EE[\eta_i^2 | \cG_i] \notag
                \\
                &\le \sum_{i = 1}^k \Bigg( \frac{\|\xb_{i}\|_{\Zb_{i - 1}^{-1}} \beta_{i-1}}{1 + \|\xb_{i}\|_{\Zb_{i - 1}^{-1}}^2} \Bigg)^2\EE[\eta_i^2 | \cG_i] \notag\\
                &\le \beta_k^2 \rho^2\sum_{i = 1}^k \EE[\eta_i^2 | \cG_i] \notag\\
                &\le \beta_k^2 \rho^2 v_k. \notag
            \end{align}
    where the first inequality holds due to Cauchy-Schwarz inequality, the second inequality holds due to the definition of indicator function $\cE_{i-1}$, the third inequality holds due to $\rho \ge \|\xb_i\|_{\Zb_{i-1}^{-1}}$ and the last inequality holds due to $\sum_{i = 1}^k \EE[\eta_i^2 | \cG_i] \leq v_k$. 
    
            Therefore, using Freedman's inequality, for any $k\ge 1$, with probability $1 - \delta / (4k^2)$, we have
            \begin{align} 
                \sum_{i = 1}^k I_{1, i} \cdot \cI_{i - 1} &\le 2\sqrt{2\beta_k^2 \rho^2 v_k \log(4k^2 / \delta)} + 4 / 3 \cdot R\beta_k \rho \log(4k^2 / \delta) \notag
                \\&\le \frac{1}{4} \beta_k^2  + 32\left(\rho^2 v_k \log(4k^2 / \delta)\right) + \frac{1}{4} \beta_k^2 + 9R^2 \rho^2 [\log(4k^2 / \delta)]^2 \notag
                \\&\le \frac{3}{4} \beta_k^2, \notag
            \end{align}
            where the first inequality holds due to Lemma \ref{lemma:freedman} and the second inequality holds due to Young's inequality. After taking a union bound for all $k>1$, it can then be further deduced that with probability $1 - \delta / 2$, for all $k \ge 1$, we have
            \begin{align} 
                \sum_{i = 1}^k I_{1, i} \cdot \cI_{i - 1} \le \frac{3}{4} \beta_k^2. \label{event:1}
            \end{align}
            For simplicity, let $\cE_{I_1}$ be the events that \eqref{event:1} holds. Then we bound the summation of $I_{2, k}$ over $k$ through the following calculation: \begin{align} 
                \sum_{i = 1}^t I_{2, k} &\le \sum_{i = 1}^k \eta_i^2 \rho^2 = \rho^2 v_k + \rho^2 \sum_{i = 1}^k \big[\eta_i^2 - \EE[\eta_i^2|\cG_i]\big] ,\label{eq:I2:1}
            \end{align}
    where the first inequality holds due to $\rho \ge \|\xb_i\|_{\Zb_{i-1}^{-1}}$.
     Still, we can bound the second term in \eqref{eq:I2:1} using Freedman's inequality in Lemma \ref{lemma:freedman}. Notice that, for each round $i\in[k]$, we have
     \begin{align}
                \left| \EE\big[\eta_i^2|\cG_i]\big] - \eta_i^2\right| \le R^2,\ & \EE\Big[\big(\EE[\eta_i^2|\cG_i] - \eta_i^2\big) |\cG_i\Big] = 0, \notag
                \\
                \EE\Big[\big(\EE[\eta_i^2|\cG_i] - \eta_i^2\big)^2 |\cG_i\Big]= \left(\EE[\eta_i^2|\cG_i]\right)^2 - & 2\EE[\eta_i^2|\cG_i]\cdot \EE[\eta_i^2 | \cG_i] + \EE[\eta_i^4|\cG_i]
                \le R^2 \EE[\eta_i^2|\cG_i], \notag
            \end{align}
            According to Freedman's Inequality, for any $k$, with probability $1 - \delta / (4k^2)$, we have \begin{align}
                \sum_{i = 1}^k \big(\eta_i^2 - \EE[\eta_i^2|\cG_i] \big) \le \sqrt{2 R^2 \log(4k^2 / \delta) v_k} + 2 / 3 \cdot R^2 \log(4k^2 / \delta).  \label{add:01}
            \end{align}
    Taking a union bound over all round $k \ge 1$,  with probability at least $1 - \delta/2$, for all $k \ge 1$, we have \begin{align} 
                \sum_{i = 1}^k I_{2, k} &\leq  \rho^2 v_k + \rho^2 \sum_{i = 1}^k \big[\eta_i^2 - \EE[\eta_i^2|\cG_i]\big] \notag\\
                &\le \rho^2 v_k + \rho^2R\sqrt{2v_k\log(4k^2 / \delta)}  + \frac{2}{3}\cdot  \rho^2 \cdot R^2 \log(4k^2 / \delta) \notag
                \\&\le 3\left(\rho^2 v_k \log(4k^2 / \delta)\right) + 2R^2 \rho^2 [\log(4k^2 / \delta)]^2 \notag
                \\&\le \frac{1}{4} \beta_k^2.  \label{eq:event2}
            \end{align}
            where the first inequality holds due to \eqref{eq:I2:1}, the second inequality holds due to
          \eqref{add:01} and the third inequality holds due to Young's inequality. For simplicity, let $\cE_{I_2}$ be the events that \eqref{eq:event2} holds.
            In the remaining proof, we assume that events $\cE_{I_1}$ and $\cE_{I_2}$ holds, whose probability is no less than $1 - \delta$ by the union bound. Under this situation, for any round $k \ge 0$, if $\cI_{i - 1}=1$ holds for all $i \in [k]$, then according to \eqref{eq:conf:recur}, we have
            \begin{align} 
                q_{k + 1} &\le \sum_{i = 1}^{k + 1}I_{1, i} + \sum_{i = 1}^{k + 1} I_{2, i} \notag
                \\&= \sum_{i = 1}^{k + 1}I_{1, i} \cdot \cI_{i - 1} + \sum_{i = 1}^{k + 1} I_{2, i} \notag
                \\&\le \beta_{k + 1}^2, \notag
            \end{align} 
    where the last inequality holds due to the definition of events $\cE_{I_1}$ and $\cE_{I_2}$. This result indicates that $\cE_{k + 1} = 1$. Therefore, by induction, we can deduce that with probability at least $1 - \delta$, for all $k\ge 1$, we have
    \begin{align} 
                \left\|\sum_{i=1}^{k} \xb_i \eta_i\right\|_{\Zb_k^{-1}} \leq \beta_k. \notag
            \end{align}
            Furthermore, the estimation error between underlying vector $\bmu^*$ and estimator $\bmu_k$ can be upper bounded by:
            \begin{align} 
                \|\bmu_k - \bmu^*\|_{\Zb_k} &= \|\Zb_k^{-1} \bbb_k - \Zb_k^{-1} \Zb_k \bmu^*\|_{\Zb_k} \notag
                \\&= \left\|\Zb_k^{-1} \bbb_k - \Zb_k^{-1} \sum_{i = 1}^k \xb_i \xb_i^\top \bmu^* - \lambda \Zb_k^{-1} \bmu^*\right\|_{\Zb_k} \notag
                \\&= \left\|\Zb_k^{-1} \sum_{i = 1}^k \xb_i (y_i - \xb_i^\top \bmu^*) - \lambda \Zb_k^{-1} \bmu^*\right\|_{\Zb_k} \notag
                \\&\le \left\|\sum_{i=1}^{k} \xb_i \eta_i\right\|_{\Zb_k^{-1}} + \sqrt{\lambda} \|\bmu^*\|_2 \notag
                \\&\le \beta_k + \sqrt{\lambda}\|\bmu^*\|_2, \notag
            \end{align}
            where the first equality follows from the definition of $\bmu_k$, the second equality holds due to the definition of $\Zb_k$ and the first inequality holds by triangle inequality with the fact that $\Zb_k \succeq \lambda\Ib$. Thus, we complete the proof of Theorem \ref{thm:bernstein1}.
    \end{proof}
    
    \section{Proofs from Section \ref{sec:bandits}}
    
    \subsection{Proof of Theorem \ref{thm:regret1}}
    \begin{lemma} \label{lemma:confidence}
        Suppose that $\|\btheta^*\|_2 \le 1$. In Algorithm \ref{alg:1}, with probability at least $1-\delta$, the
    following statement holds for all round $k \ge 1$ and layer $\ell \in [L]$: \begin{align} 
            \|\hat \btheta_{k, \ell} - \btheta^*\|_{\hat\bSigma_{k, \ell}} \le 16 \cdot 2^{-\ell} \sqrt{\sum_{i \in \Psi_{k, \ell}} w_i^2 \sigma_i^2\log(4k^2 L / \delta)} + 6 \cdot 2^{-\ell} R \log(4k^2 L / \delta) + 2^{-\ell + 1}. \notag
        \end{align}
    \end{lemma}
    For simplicity, we denote $\cE_{\conf}$ as the event such that the result in Lemma \ref{lemma:confidence} holds in the remaining section. 
    \begin{proof}
        We first consider a fixed layer $\ell \in [L]$. 
        Suppose that $k$ is an arbitrary round satisfying $k \in \Psi_{k + 1, \ell}$. 
        Notice that in Line \ref{alg:1:line:15} (Algorithm \ref{alg:1}), we introduce weight $w_k$ to guarantee $\|w_k \ab_k\|_{\hat\bSigma_{k, \ell}^{-1}} = 2^{-\ell}$.

    Then we can applying Theorem \ref{thm:bernstein1} for the layer $\ell$. In detail, for each $k \in \Psi_{K + 1, \ell}$, we have \begin{align} 
            \|w_k \ab_k\|_{\hat\bSigma_{k, \ell}^{-1}} = 2^{-\ell}, \quad
            \EE[w_k^2 \epsilon_k^2| \cF_{k}] \le w_k^2 \EE[\epsilon_k^2| \cF_{k}] \le w_k^2 \sigma_k^2, \quad
            |w_k \epsilon_k| \le |\epsilon_k| \le R, \notag
        \end{align}
        where the last inequality holds due to the fact that
    $w_k = 2^{-\ell}/\|\ab_k\|_{\hat\bSigma_{k, \ell}^{-1}}\leq 1$. According to Theorem \ref{thm:bernstein1}, we can deduce that  with probability at least $1 - \delta / L$, for all round $k \in \Psi_{K + 1, \ell}, \ $\begin{align} 
             \|\hat \btheta_{k, \ell} - \btheta^*\|_{\hat\bSigma_{k, \ell}} \le 16 \cdot 2^{-\ell} \sqrt{\sum_{i \in \Psi_{k, \ell}} w_i^2 \sigma_i^2\log(4k^2 L / \delta)} + 6 \cdot 2^{-\ell} R \log(4k^2 L / \delta) + 2^{-\ell + 1}. \notag
        \end{align}
    Finally, after taking a union bound for all layer $\ell \in [L]$, we complete the proof of \ref{lemma:confidence}.
    \end{proof}
    
    \begin{lemma}  \label{lemma:confset}
        Suppose that the event $\cE_{\conf}$ defined in Lemma \ref{lemma:confidence} occurs. If $\{\hat\beta_{k, \ell}\}_{k \ge 1, \ell \in [L]}$ satisfies \begin{align}
            \hat\beta_{k, \ell} \ge 16 \cdot 2^{-\ell} \sqrt{\sum_{i \in \Psi_{k, \ell}} w_i^2 \sigma_i^2\log(4k^2 L / \delta)} + 6 \cdot 2^{-\ell} R \log(4k^2 L / \delta) + 2^{-\ell + 1}, \notag
        \end{align}  
        then for all $k \ge 1$ and $\ell \in [L]$ such that $\cA_{k, \ell}$ exists, we have $\ab_k^* \in \cA_{k, \ell}$.
    \end{lemma}
    \begin{proof} 
        Fix an arbitrary round $k$. 
        If layer $\ell = 1$, then $\ab_k^* \in \cD_k= \cA_{k, \ell}$ trivially holds. 
        Then for layer $\ell >1$,
        we prove lemma \ref{lemma:confset}
        by induction. Assume that $\ab_k^* \in \cA_{k, \ell_1}$ holds for some $\ell_1 \in \ZZ^+$ and $\cA_{k, \ell_1 + 1}$ exists. 
    
        By Lemma \ref{lemma:confidence}, for all $\ab \in \cA_{k, \ell_1}$, we have
        \begin{align} 
            \left|\la \ab, \hat\btheta_{k, \ell_1} \ra - \la \ab, \btheta^* \ra\right| \le \|\ab\|_{\hat\bSigma_{k, \ell_1}^{-1}} \left\|\hat\btheta_{k, \ell_1} - \btheta^*\right\|_{\hat\bSigma_{k, \ell_1}} \le \hat\beta_{k, \ell} \|\ab\|_{\hat\bSigma_{k, \ell_1}^{-1}}, \label{eq:reward:gap}
        \end{align}
        where the first inequality holds due to Cauchy-Schwarz inequality and the last inequality holds due to the definition of events $\cE_{\conf}$. 
        According to Line \ref{alg1:line:11} of Algorithm \ref{alg:1}, $\cA_{k, \ell_1 + 1}$ exists only if $\|\ab\|_{\hat\bSigma_{k, \ell_1}^{-1}} \le 2^{-\ell_1}$ holds for all $\ab \in \cA_{k, \ell_1}$. Therefore, the sub-optimality gap in \eqref{eq:reward:gap} can be further bounded as follows: \begin{align} 
            \left|\la \ab, \hat\btheta_{k, \ell_1} \ra - \la \ab, \btheta^* \ra\right| \le \hat\beta_{k, \ell} \|\ab\|_{\hat\bSigma_{k, \ell_1}^{-1}} \le 2^{-\ell_1} \cdot \hat\beta_{k, \ell_1}. \label{add:002}
        \end{align}
        For short, let $\ab_{\max} = \argmax_{\ab' \in \cA_{k, \ell_1}} \la \ab', \hat\btheta_{k, \ell_1} \ra$. 
    Then for the optimal action $\ab_k^*\in \cA_{k,l_1}$, we have \begin{align} 
            \ &\la \ab_k^*, \hat\btheta_{k, \ell_1} \ra - \max_{\ab' \in \cA_{k, \ell_1}} \la \ab', \hat\btheta_{k, \ell_1} \ra  \notag
            \\
            &= \la \ab_k^*, \hat\btheta_{k, \ell_1} \ra - \la \ab_{\max}, \hat\btheta_{k, \ell_1} \ra \notag\\
            &\ge \la \ab_k^*, \btheta^* \ra - \la \ab_{\max}, \btheta^* \ra - \left|\la \ab_k^*, \hat\btheta_{k, \ell_1} \ra - \la \ab_k^*, \btheta^* \ra\right| - \left|\la \ab_{\max}, \hat\btheta_{k, \ell_1} \ra - \la \ab_{\max}, \btheta^* \ra\right| \notag
            \\&\ge -2^{-\ell_1 + 1} \cdot \hat\beta_{k, \ell_1}, \notag
        \end{align}
        where the last inequality holds due to \eqref{add:002}
     with the fact that $ \la \ab_k^*, \btheta^* \ra \ge \la \ab_{\max}, \btheta^* \ra $. Therefore, according to the Line \ref{alg1:line:12} (Algorithm \ref{alg:1}), the optimal action $\btheta^* \in \cA_{k, \ell_1 + 1}$. Therefore, by induction, we complete the proof of Lemma \ref{lemma:confset}
    \end{proof}
    
    \begin{lemma} \label{lemma:regret-layer}
        Suppose for all $k \ge 1$ and all $\ab \in \cD_k$, we have $\|\ab\|_2 \le A, \|\btheta^*\|_2 \le 1$. If $\cE_{\conf}$ occurs and $\{\beta_{k, \ell}\}_{k\ge 1, \ell \in [L]}$ satisfies the requirement in Lemma \ref{lemma:confset}, then for all $\ell \in [L] \backslash \{1\}$, the regret incurred by the index set $\Psi_{T + 1, \ell}$ is bounded as follows : 
        \begin{align} 
            \sum_{\tau \in \Psi_{K + 1, \ell}} \big(\la \ab_{\tau}^*, \btheta^* \ra - \la \ab_{\tau}, \btheta^* \ra\big) \le \tilde{O}\left(d \cdot 2^{\ell} \cdot \hat\beta_{K, \ell - 1}\right). \notag
        \end{align}
    \end{lemma}
    
    \begin{proof} 
        For all round $\tau \in \Psi_{K + 1, \ell}$, 
        we can deduce that $\ab_\tau, \ab_\tau^* \in \cA_{\tau, \ell}$ by Lemma \ref{lemma:confset}. Also, according to Line \ref{alg1:line:12} of Algorithm \ref{alg:1}, we have \begin{align}\la \ab_{\tau}^*, \hat \btheta_{\tau, \ell - 1} \ra - \la \ab_{\tau}, \hat \btheta_{\tau, \ell - 1} \ra \le 2^{-\ell + 2} \hat\beta_{\tau, \ell - 1}. \label{eq:b3b2} \end{align}
       Besides, from Line \ref{alg1:line:11} and the round  $\tau \in \Psi_{K+1, \ell}$, we have \begin{align}
            \|\ab_\tau\|_{\hat \bSigma_{\tau, \ell - 1}^{-1}} \le 2^{-\ell + 1}, \quad \|\ab_\tau^*\|_{\hat \bSigma_{\tau, \ell - 1}^{-1}} \le 2^{-\ell + 1}. \label{eq:b3b3}
        \end{align} 
        We further compute  \begin{align} 
            \la \ab_{\tau}^*, \btheta^*\ra - \la \ab_{\tau}, \btheta^*\ra &\le \la \ab_{\tau}^*, \hat \btheta_{\tau, \ell - 1} \ra + \left|\la \ab_{\tau}^*, \hat \btheta_{\tau, \ell - 1} - \btheta^* \ra\right| - \la \ab_{\tau}, \hat \btheta_{\tau, \ell - 1}\ra + \left|\la \ab_{\tau}, \hat \btheta_{\tau, \ell - 1} - \btheta^* \ra\right| \notag
            \\
            &\leq \la \ab_{\tau}^*, \hat \btheta_{\tau, \ell - 1} \ra - \la \ab_{\tau}, \hat \btheta_{\tau, \ell - 1}\ra \notag\\
            &\qquad +\|\ab_{\tau}^*\|_{\hat\bSigma_{\tau, \ell-1}^{-1}} \left\|\hat\btheta_{\tau, \ell-1} - \btheta^*\right\|_{\hat\bSigma_{\tau, \ell-1}} +\|\ab_{\tau}\|_{\hat\bSigma_{\tau, \ell-1}^{-1}} \left\|\hat\btheta_{\tau, \ell-1} - \btheta^*\right\|_{\hat\bSigma_{\tau, \ell-1}}
            \notag\\
            &\le 2^{-\ell + 2} \cdot \hat \beta_{\tau, \ell - 1} + 2^{-\ell + 1} \cdot \hat \beta_{\tau, \ell - 1} + 2^{-\ell + 1} \cdot \hat \beta_{\tau, \ell - 1} \notag
            \\&= 8 \cdot 2^{-\ell} \cdot \hat \beta_{\tau, \ell - 1}, \label{add:03}
        \end{align}
        where the second inequality holds due to Cauchy-Schwarz inequality and
        the last inequality holds due to Lemma \ref{lemma:confidence}, \eqref{eq:b3b2} and \eqref{eq:b3b3}. 
        Taking the summation over $\tau \in \Psi_{K+ 1, \ell}$, we have \begin{align} 
            \sum_{\tau \in \Psi_{K + 1, \ell}} \big(\la \ab_{\tau}^*, \btheta^* \ra - \la \ab_{\tau}, \btheta^* \ra\big) &\le  8 \cdot 2^{-\ell} \cdot \hat \beta_{K, \ell - 1} \left|\Psi_{K + 1, \ell}\right| \notag
            \\&\le 8 \cdot 2^{\ell} \cdot \hat\beta_{K, \ell - 1}  \cdot \sum_{k \in \Psi_{K + 1, \ell}} \|w_k \cdot \ab_k\|_{\hat\bSigma_{k, \ell}^{-1}}^2 \notag
            \\&\le 8 \cdot 2^{\ell} \cdot \hat\beta_{K, \ell - 1} \cdot 2 d \log\left(1 + 2^{2\ell}K \cdot A^2 / d\right),\notag
        \end{align}
        where the first inequality holds due to \eqref{add:03}, the second inequality holds since for all round $k \in \Psi_{k + 1, \ell}$, the weight $w_k$ satisfies $\|w_k \ab_k\|_{\hat\bSigma_{k, \ell}^{-1}} = 2^{-\ell}$, and the last inequality holds due to Lemma \ref{Lemma:abba}. 
    
    \end{proof} 
    \begin{lemma} \label{lemma:var}
        Let weight $w_i$ be defined in Algorithm \ref{alg:1}. 
        With probability at least $1 - 2\delta$, for all $k \ge 1$, $\ell \in [L]$, the following two inequalities hold simultaneously: \begin{align*}
            \sum_{i \in \Psi_{k + 1, \ell}} w_i^2 \sigma_i^2 \le 2 \sum_{i \in \Psi_{k + 1, \ell}} w_i^2 \epsilon_i^2 + \frac{14}{3} R^2 \log(4k^2 L / \delta), 
            \\
            \sum_{i \in \Psi_{k + 1, \ell}} w_i^2 \epsilon_i^2 \le \frac{3}{2} \sum_{i \in \Psi_{k + 1, \ell}} w_i^2 \sigma_i^2 + \frac{7}{3} R^2 \log(4k^2 L / \delta). 
        \end{align*}
    \end{lemma}
    For simplicity, we denote $\cE_{\var}$ as the event such that the two inequalities in Lemma \ref{lemma:var} holds. 
    \begin{proof} 
          We first consider a fixed layer $\ell \in [L]$. 
        For the gap between $\sum\limits_{i \in \Psi_{k + 1, \ell}} w_i^2 \sigma_i^2$ and $\sum\limits_{i \in \Psi_{k + 1, \ell}} w_i^2 \epsilon_i^2$, according to the definition, we have
        \begin{align*} 
            \text{for}\ \forall i \ge 1,\ \ &\EE\left[\epsilon_i^2 - \sigma_i^2 | \ab_{1:i}, r_{1:i - 1}\right] = 0, \\
            \sum_{i \in \Psi_{k + 1, \ell}} \EE\left[w_i^2 (\epsilon_i^2 - \sigma_i^2)^2| \ab_{1:i}, r_{1:i - 1}\right] &\le \sum_{i \in \Psi_{k + 1, \ell}}  \EE\left[w_i^2\epsilon_i^4| \ab_{1:i}, r_{1:i - 1}\right] \le R^2 \sum_{i \in \Psi_{k + 1, \ell}} w_i^2 \sigma_i^2, 
        \end{align*}
        where the first inequality holds due to $\text{Var}[x]\leq \EE[x^2]$ and the second inequality holds due to $|\epsilon_i|\leq R$ and $\EE\left[\epsilon_i^2| \ab_{1:i}, r_{1:i - 1}\right]= \sigma_i^2$.
        Applying Freedman's inequality (Lemma \ref{lemma:freedman}) with $\{\epsilon_i^2\}_{i \in  \Psi_{k + 1, \ell}}$ and taking a union bound for all $k\ge 1$
        , with probability at least $1 - 2\delta / L$, for all $k \ge 1$, the following inequality holds 
        \begin{align*} 
            \left|\sum_{i \in \Psi_{k + 1, \ell}} w_i^2 (\sigma_i^2 - \epsilon_i^2)\right| &\le \sqrt{2R^2 \sum_{i \in \Psi_{k + 1, \ell}} w_i^2 \sigma_i^2 \log(4k^2 L / \delta)} + \frac{2}{3} \cdot 2R^2 \log(4k^2 L / \delta) 
            \\&\le \frac{1}{2} \sum_{i \in \Psi_{k + 1, \ell}} w_i^2 \sigma_i^2 + \frac{7}{3}R^2 \log(4k^2 L / \delta),  
        \end{align*}
    where the last inequality holds due to Young's inequality.
        Rearranging the above inequality, we conclude that $\PP(\cE_{\var}) \ge 1 - 2\delta$ by applying union bound over all $\ell \in [L]$. Thus, we complete the proof of Lemma \ref{lemma:var}.
    \end{proof}
    
    \begin{lemma} \label{lemma:varest}
        Suppose that $\|\btheta^*\|_2 \le 1$. Let weight $w_i$ be defined in Algorithm \ref{alg:1}. 
        On the event $\cE_{\conf}$ and $\cE_{\var}$ (defined in Lemma \ref{lemma:confidence}, \ref{lemma:var}), for all $k \ge 1$, $\ell \in [L]$ such that $2^\ell \ge 64 \sqrt{\log\left(4(k + 1)^2 L /\delta\right)}$, we have the following inequalities: 
        \begin{align*} 
            \sum_{i \in \Psi_{k + 1, \ell}} w_i^2 \sigma_i^2 \le 8 \sum_{i \in \Psi_{k + 1, \ell}} w_i^2 \left(r_i - \la \hat\btheta_{k + 1, \ell}, \ab_i \ra \right)^2 + 6R^2 \log(4(k + 1)^2 L / \delta) + 2^{-2\ell + 4}, \\
            \sum_{i \in \Psi_{k + 1, \ell}} w_i^2 \left(r_i - \la \hat\btheta_{k + 1, \ell}, \ab_i \ra \right)^2 \le \frac{3}{2} \sum_{i \in \Psi_{k + 1, \ell}} w_i^2 \sigma_i^2 + \frac{7}{3} R^2 \log(4k^2 L / \delta) + 2^{-2\ell}. 
        \end{align*}
    \end{lemma}
    
    \begin{proof} 
        Let $\ell$ be an arbitrary index in $[L]$. 
        By the definition of events $\cE_{\var}$, we have \begin{align} 
            \sum_{i \in \Psi_{k + 1, \ell}} w_i^2 \sigma_i^2 &\le 2 \sum_{i \in \Psi_{k + 1, \ell}} w_i^2 \epsilon_i^2 + \frac{14}{3} R^2 \log(4k^2 L / \delta) \notag
            \\&\le 4 \sum_{i \in \Psi_{k + 1, \ell}} w_i^2 \left(r_i - \la \hat\btheta_{k + 1, \ell}, \ab_i \ra \right)^2 + 4\sum_{i \in \Psi_{k + 1, \ell}} w_i^2 \left[\epsilon_i - \left(r_i - \la \hat\btheta_{k + 1, \ell}, \ab_i \ra\right)\right]^2 \notag \\&\quad + \frac{14}{3} R^2 \log(4k^2 L / \delta), \label{eq:varest:0}
        \end{align}
        where the last inequality holds due to $(a + b)^2 \le 2a^2 + 2b^2$.  In addition, the gap between $\epsilon_i$ and $r_i - \la \hat\btheta_{k + 1, \ell}, \ab_i \ra$ can be upper bounded by \begin{align} 
            &\sum_{i \in \Psi_{k + 1, \ell}} w_i^2 \left[\epsilon_i - \left(r_i - \la \hat\btheta_{k + 1, \ell}, \ab_i \ra\right)\right]^2\notag\\
            &= \sum_{i \in \Psi_{k + 1, \ell}} w_i^2 \left(\la \hat\btheta_{k + 1, \ell} - \btheta^*, \ab_i \ra\right)^2 \notag
            \\&= \sum_{i \in \Psi_{k + 1, \ell}}\left(\hat\btheta_{k + 1, \ell} - \btheta^*\right)^\top (w_i \ab_i) \cdot (w_i \ab_i)^\top \left(\hat\btheta_{k + 1, \ell} - \btheta^*\right) \notag
            \\&\le \left(\hat\btheta_{k + 1, \ell} - \btheta^*\right)^\top \hat \bSigma_{k + 1, \ell} \left(\hat\btheta_{k + 1, \ell} - \btheta^*\right) \notag
            \\&\le \Big( 16 \cdot 2^{-\ell} \sqrt{\sum_{i \in \Psi_{k + 1, \ell}} w_i^2 \sigma_i^2\log(4(k + 1)^2 L / \delta)}  + 6 \cdot 2^{-\ell} R \log(4(k + 1)^2 L / \delta) + 2^{-\ell + 1}\Big)^2, \label{eq:varest:1}
        \end{align}
        where the first inequality holds due to $\hat \bSigma_{k + 1, \ell} \succeq w_i^2\ab_i\ab_i^{\top}$
    and the last inequality holds due to Lemma~ 
     \ref{lemma:confidence}. 
        From \eqref{eq:varest:1}, when $2^\ell \ge 64 \sqrt{\log\left(4(k + 1)^2 L /\delta\right)}$, we have\begin{align} 
            \sum_{i \in \Psi_{k + 1, \ell}} w_i^2 &\left[\epsilon_i - \left(r_i - \la \hat\btheta_{k + 1, \ell}, \ab_i \ra\right)\right]^2 \notag \\&\le \frac{1}{8} \sum_{i \in \Psi_{k + 1, \ell}} w_i^2 \sigma_i^2 + 2\left(6 \cdot 2^{-\ell} R \log(4(k + 1)^2 L / \delta) + 2^{-\ell + 1}\right)^2. \label{eq:varest:2}
        \end{align} 
    where the inequality holds due to \eqref{eq:varest:1} with the fact that $(a+b)^2\leq 2a^2 +2 b^2$.
        Substituting \eqref{eq:varest:2} into \eqref{eq:varest:0}, we have \begin{align*}
            \sum_{i \in \Psi_{k + 1, \ell}} w_i^2 \sigma_i^2 &\le 4 \sum_{i \in \Psi_{k + 1, \ell}} w_i^2 \left(r_i - \la \hat\btheta_{k + 1, \ell}, \ab_i \ra \right)^2 + 4\sum_{i \in \Psi_{k + 1, \ell}} w_i^2 \left[\epsilon_i - \left(r_i - \la \hat\btheta_{k + 1, \ell}, \ab_i \ra\right)\right]^2 \notag \\&\quad + \frac{14}{3} R^2 \log(4k^2 L / \delta)
            \\&\le 4 \sum_{i \in \Psi_{k + 1, \ell}} w_i^2 \left(r_i - \la \hat\btheta_{k + 1, \ell}, \ab_i \ra \right)^2 + \frac{1}{2}  \sum_{i \in \Psi_{k + 1, \ell}} w_i^2 \sigma_i^2 \\&\quad + 2\left(6 \cdot 2^{-\ell} R \log(4(k + 1)^2 L / \delta) + 2^{-\ell + 1}\right)^2 + \frac{14}{3} R^2 \log(4k^2 L / \delta)
            \\&\le 8 \sum_{i \in \Psi_{k + 1, \ell}} w_i^2 \left(r_i - \la \hat\btheta_{k + 1, \ell}, \ab_i \ra \right)^2 + 6R^2 \log(4(k + 1)^2 L / \delta) + 2^{-2\ell + 4},
        \end{align*}
        where the last inequality holds due to the fact that $x\leq x/2 +y$ implies $x\leq 2y$. Thus, we complete the proof of the first part of Lemma \ref{lemma:varest}. 
    
        For the second part, note that $\btheta_{k + 1, \ell}$ is the minimizer of the following weighted ridge regression 
        \begin{align*} 
           \btheta_{k + 1, \ell} \leftarrow \arg\min_{\btheta \in \RR^d} \sum_{i \in \Psi_{k + 1, \ell}} w_i^2 \big(r_i - \la \btheta, \ab_i \ra\big)^2 + 2^{-2\ell} \|\btheta\|_2^2. 
        \end{align*}
        Thus, we have \begin{align*} 
            \sum_{i \in \Psi_{k + 1, \ell}} w_i^2 \left(r_i - \la \hat\btheta_{k + 1, \ell}, \ab_i \ra \right)^2 \le \sum_{i \in \Psi_{k + 1, \ell}} w_i^2 \big(r_i - \la \btheta^*, \ab_i \ra \big)^2 + 2^{-2\ell} \|\btheta^*\|_2^2
            \leq 
            \sum_{i \in \Psi_{k + 1, \ell}} w_i^2 \epsilon_i^2 + 2^{-2\ell}, 
        \end{align*}
        where the second inequality holds due to $\|\btheta^*\|_2\leq 1$. Combining the result in Lemma \ref{lemma:var}, we can further conclude that
     \begin{align*} 
            \sum_{i \in \Psi_{k + 1, \ell}} w_i^2 \left(r_i - \la \hat\btheta_{k + 1, \ell}, \ab_i \ra \right)^2 &\leq  \sum_{i \in \Psi_{k + 1, \ell}} w_i^2 \epsilon_i^2 + 2^{-2\ell}
            \notag\\
            &\le \frac{3}{2} \sum_{i \in \Psi_{k + 1, \ell}} w_i^2 \sigma_i^2 + \frac{7}{3} R^2 \log(4k^2 L / \delta) + 2^{-2\ell}. 
        \end{align*}
        Thus, we complete the proof of Lemma \ref{lemma:varest}.
    \end{proof}

    \begin{proof}[Proof of Theorem \ref{thm:regret1}]
        Applying a union bound on event $\cE_{\conf}$and $\cE_{\var}$ defined in Lemma \ref{lemma:confidence} and \ref{lemma:var}, we have $P(\cE_{\conf} \cap \cE_{\var}) \ge 1 - 3\delta$. In the remaining proof, we suppose that $\cE_{\conf}, \cE_{\var}$ hold simultaneously. 
        For simplicity, let $\ell^* = \lceil \frac{1}{2} \log_2 \log\left(4(K + 1)^2 L /\delta\right) \rceil + 8$. 
     By  Lemma \ref{lemma:regret-layer} and Lemma \ref{lemma:varest}, we have for all $\ell \in [L] \backslash [\ell^*]$, \begin{align}
            \hat\beta_{K, \ell-1} \ge 16 \cdot 2^{-(\ell-1)} \sqrt{\sum_{i \in \Psi_{K, \ell-1}} w_i^2 \sigma_i^2\log(4K^2 L / \delta)} + 6 \cdot 2^{-\ell} R \log(4K^2 L / \delta) + 2^{-\ell}, \notag
        \end{align}
     which further implies
     \begin{align} 
            \sum_{\tau \in \Psi_{K + 1, \ell}} \left(\la \ab_{\tau}^*, \btheta^* \ra - \la \ab_{\tau}, \btheta^* \ra\right) &\le \tilde{O}\left(d \cdot 2^{\ell} \cdot \hat\beta_{K, \ell - 1}\right) \notag
            \\&\le \tilde{O} \left(d \sqrt{\sum_{k = 1}^K w_k^2 \left(r_k - \la \hat\btheta_{K + 1, \ell}, \ab_k \ra \right)^2 + R^2 + 1} + R\right) \notag
            \\&\le \tilde{O} \left(d \sqrt{\sum_{k = 1}^K \sigma_k^2} + dR + d\right), \label{eq:regret1:0}
        \end{align}
        where the first inequality holds due to Lemma \ref{lemma:regret-layer},
        the second inequality holds due to \eqref{eq:def:beta} and the last inequality follows from Lemma \ref{lemma:varest}. 
    
        For each round $k \in [K] \backslash \left(\bigcup_{\ell \in [L]} \Psi_{K + 1, \ell} \right) := \Psi_{K + 1, L + 1}$, we set $\ell_k$ as the value of layer $\ell$ such that the while loop in Algorithm \ref{alg:1} stops. Therefore, we have   \begin{align} 
            \sum_{k \in [K] \backslash \left(\bigcup_{\ell \in [L]} \Psi_{K + 1, \ell} \right)} \big(\la \ab_{k}^*, \btheta^* \ra - \la \ab_{k}, \btheta^* \ra\big)
            &\le \sum_{k \in \Psi_{K + 1, L + 1}} \left(\la \ab_{k}, \hat \btheta_{k, \ell_k} \ra + \hat \beta_{k, \ell_k} \cdot \alpha - \la \ab_{k}, \btheta^* \ra\right) \notag
            \\&\le \sum_{k \in \Psi_{K + 1, L + 1}} \left(\hat \beta_{k, \ell_k} \cdot \alpha + \alpha \cdot \|\btheta^* - \hat \btheta_{k, \ell_k}\|_{\hat\bSigma_{k, \ell_k}}\right) \notag
            \\&\le \sum_{k \in \Psi_{K + 1, L + 1}} 2 \alpha \cdot \hat \beta_{k, \ell_k} \notag
            \\&\le K \cdot \tilde{O} \left(1 / K\right) = \tilde{O}(1), \label{eq:regret1:1}
        \end{align}
        where the first inequality holds due to the selection rule of action $\ab_k$ (Line \ref{line:selection} in Algorithm \ref{alg:1}) with Lemma \ref{lemma:confidence}, Lemma \ref{lemma:varest} and the fact that $\ab_k^* \in \cA_{k,\ell_k}$ (Lemma \ref{lemma:confset}), the second inequality holds due to Cauchy-Schwarz inequality, the third inequality follows from Lemma \ref{lemma:confidence} and the last inequality follows from the definition of $\alpha$.
    
        Finally, for layer $\ell \in [\ell^*]$ and round $\tau \in \Psi_{K + 1, \ell}$, we have \begin{align} 
            \sum_{\tau \in \Psi_{K + 1, \ell}} \big(\la \ab_{\tau}^*, \btheta^* \ra - \la \ab_{\tau}, \btheta^* \ra\big) \le 2 \left|\Psi_{K + 1, \ell}\right| =  2^{2\ell+1} \sum_{\tau \in \Psi_{K + 1, \ell}} \|w_\tau\ab_\tau\|_{\hat \bSigma_{\tau, \ell}}^2 \le \tilde{O}(d), \label{eq:regret1:2}
        \end{align}
        where the first inequality holds since the reward is in the range $[-1,1]$, the equation follows from the fact that $\|w_\tau\ab_\tau\|_{\hat \bSigma_{\tau, \ell}}=2^{-\ell}$ holds for all $\tau \in \Psi_{K + 1, \ell}$
        and the last inequality follows from Lemma \ref{Lemma:abba} with the fact that $2^{\ell^*} \le 128\sqrt{\log(4(K + 1)^2 L /\delta)}$ is bounded by a logarithmic term. 
        Putting \eqref{eq:regret1:0}, \eqref{eq:regret1:1}, \eqref{eq:regret1:2} together, we have \begin{align*} 
            \regret(K) \le \tilde{O}\left(d \sqrt{\sum_{k = 1}^K \sigma_k^2} + dR  + d\right). 
        \end{align*}
        Thus, we complete the proof of Theorem \ref{thm:regret1}.
    \end{proof}

    \section{Proofs from Section \ref{sec-4}}
    For $k \in [K]$, $h \in [H]$, let $\cF_{k,h}$ be the $\sigma$-algebra generated by the random variables representing the state-action pairs up to and including those that appear stage $h$ of episode $k$. More specifically, $\cF_{k,h}$ is generated by
    \begin{align*}
    s_1^1,a_1^1, \dots, s_h^1,a_h^1, &\dots, s_H^1,a_H^1\,, \\
    s_1^2,a_1^2, \dots, s_h^2,a_h^2, &\dots, s_H^2,a_H^2\,, \\
    \vdots \\
    s_1^k,a_1^k,\dots, s_h^k,a_h^k & \,.
    \end{align*}
    For simplicity, we define the following indicator sequence $I_h^k$  for all $(k, h) \in [K] \times [H]$
    : \begin{align} 
        I_h^k = \mathds{1}\left\{\forall \ell \in [L], \det\left(\hat\bSigma_{k, h, \ell}\right) / \det\left(\hat\bSigma_{k, 1, \ell}\right) \le 4\right\}. \label{eq:def:ind}
    \end{align}
    For each $1\leq h_1 \leq h_2 \leq H$, since $\hat\bSigma_{k, h_2, \ell} \succeq \hat\bSigma_{k, h_1, \ell}$, the indicator function is monotonic (e.g., $I_{h_1}^k\leq I_{h_2}^k$). In addition, the following lemma provides an upper bound for the number of episodes when the determinant of covariance matrix grows sharply. 
    \begin{lemma} \label{lemma:g}
        If the indicator function $I_h^k$ is defined as in \eqref{eq:def:ind}, then for each $k\in[K]$, we have \begin{align*}
            \sum_{i = 1}^k (1 - I_H^i) \le \frac{dL}{2}  \log \frac{\lambda + k H /d}{\lambda} + d L^2. 
        \end{align*}
    \end{lemma}
    \begin{proof} 
        For all layer $\ell \in [L]$, let $\cD_\ell$ be the set of indices $i \in [k]$ such that \begin{align*}\det\left(\hat\bSigma_{i + 1, 1, \ell}\right) / \det\left(\hat\bSigma_{i, 1, \ell}\right) > 4. \end{align*}
    According to the update rule of $\bSigma_{k,1,\ell}$,   $\bSigma_{k+1,1,\ell}\succeq\bSigma_{k,1,\ell}$ holds for all episode $k\in[K]$.
    Therefore, we have
    \begin{align}
        \det(\hat \bSigma_{k + 1, 1, \ell})/\det(\hat \bSigma_{1, 1, \ell})=
        \prod_{i=1}^k \det\left(\hat\bSigma_{i + 1, 1, \ell}\right) / \det\left(\hat\bSigma_{i, 1, \ell}\right) \ge 4^{|\cD_\ell|},\label{add:004}
    \end{align}
    where the inequality holds due to the definition of set $\cD_\ell$. In addition, the determinant of matrices $\hat \bSigma_{k + 1, 1, \ell}$ and $\hat \bSigma_{1, 1, \ell}$ is bounded by: 
         \begin{align*}
            \det(\hat \bSigma_{k + 1, 1, \ell})
            &\le \left(\text{tr}(\bSigma_{k + 1, 1, \ell}) / d\right)^d
            \le (2^{-2\ell}\lambda + k H /d)^d, \\
            \det(\hat \bSigma_{1, 1, \ell}) &= \big(2^{-2\ell} \cdot \lambda\big)^d, 
        \end{align*}
        where the first inequality holds since $\hat \bSigma_{k + 1, 1, \ell} \succeq \zero$, the
        last inequality holds due to $w_{k,i}\leq 1$ and $
        \|\bphi_{V_{i,h+1}}(s_h^i,a_h^i)\|_2\leq 1$. Combining these results,
        it holds that \begin{align*} 
            |\cD_\ell| \le \log_4 \left(\frac{(\lambda + 2^{2\ell}k H /d)^d}{\lambda^d}
            \right) \le \frac{d}{2} \log_2 \frac{\lambda + 2^{2\ell}k H /d}{\lambda} \le \frac{d}{2} \log_2\frac{\lambda + k H /d}{\lambda} + d \cdot \ell. 
        \end{align*}
    Finally, according to the definition of $\cD_\ell$ and 
    indicator function $I_h^k$, we have
    \begin{align*}
        \sum_{i = 1}^k (1 - I_H^i) \le \sum_{\ell \in [L]} |\cD_\ell| \le \frac{dL}{2}  \log_2 \frac{\lambda + k H /d}{\lambda} + dL^2.
    \end{align*}
    Thus, we complete the proof of Lemma \ref{lemma:g}.
    \end{proof}
    
    \begin{lemma}\label{lemma:card}
        Let $\Psi_{K + 1, \ell}$ be defined in \eqref{eq:def:index}. Then for all layer $\ell \in [L]$, it holds that $|\Psi_{K + 1, \ell}| \le 2d \log \big(1 + KH/({2^{-2\ell}d\lambda})\big)$. 
    \end{lemma}
    
    \begin{proof} 
        By the definition of $w_{k, h}$ in Algorithm \ref{alg:2}, \begin{align*} 
            \sum_{(k, h) \in \Psi_{K + 1, \ell}} \|w_{k, h} \bphi_{V_{k, h + 1}}(s_h^k, a_h^k)\|_{\hat \bSigma_{k, h, \ell}^{-1}}^2 = |\Psi_{K + 1, \ell}| \cdot 2^{-2\ell}. 
        \end{align*}
        On the other hand, by Lemma \ref{Lemma:abba}, we have \begin{align*} 
            \sum_{(k, h) \in \Psi_{K + 1, \ell}} \|w_{k, h} \bphi_{V_{k, h + 1}}(s_h^k, a_h^k)\|_{\hat \bSigma_{k, h, \ell}^{-1}}^2 \le 2d \log \frac{2^{-2\ell} d\lambda + KH}{2^{-2\ell}d\lambda}. 
        \end{align*}
        Combining these results, we further conclude that $|\Psi_{K + 1, \ell}| \le 2\cdot d \log \big(1 + KH/({2^{-2\ell}d\lambda})\big)$. Thus, we complete the proof of Lemma \ref{lemma:g}.
    \end{proof}
    
    \subsection{High-Probability Events}
    For simplicity, we define the stochastic transition noise $\epsilon_{k,h}$ and variance $\sigma_{k,h}$ as follows:
     \begin{align} 
        \epsilon_{k, h} &= V_{k, h + 1}(s_{h + 1}^k) - \left\la  \btheta^*, \bphi_{V_{k, h + 1}}(s_h^k, a_h^k) \right\ra, \notag \\
        \sigma_{k, h} &= \sqrt{\left[\VV V_{k, h + 1}\right](s_h^k, a_h^k)}. \label{eq:def:sigma}
    \end{align}
    With these notations, we further define the following high-probability events: 
    \begin{align} 
        &\cE_{\mathrm{c}} = \Big\{\forall k \ge 1, \ell \in [L], \|\hat \btheta_{k, \ell} - \btheta^* \|_{\hat \bSigma_{k, 1, \ell}} \le 16 \cdot 2^{-\ell} \sqrt{\sum_{(i, h) \in \Psi_{k, \ell}} w_{i, h}^2 \sigma_{i, h}^2 \log(4k^2H^2 L / \delta)} \notag 
        \\&\qquad \quad + 6 \cdot 2^{-\ell} \log(4k^2H^2 L /\delta) + 2^{-\ell} \sqrt{\lambda} \cdot B\Big\}, \label{eq:event:sub}
        \\
        &\cE_{\mathrm{var}'} = \left\{\forall k \ge 1, \sum_{(i, h) \in \Psi_{k, \ell}} w_{i, h}^2 \left|\epsilon_{i, h}^2 - \sigma_{i, h}^2\right| \le \frac{1}{2} \sum_{(i, h) \in \Psi_{k, \ell}} w_{i, h}^2 \sigma_{i, h}^2 + \frac{7}{3} \log\left(4k^2H^2 / \delta\right)\right\}.  \label{eq:event:var'}
    \end{align}
    
    \begin{lemma} \label{lemma:event:sub}
        Let $\cE_{\mathrm{c}}$ be defined in \eqref{eq:event:sub}. Then we have $\PP(\cE_{\mathrm{c}}) \ge 1 - \delta$. 
    \end{lemma}
    
    \begin{proof} 
        From the definition of $\ell_{k, h}$ and $w_{k, h}$ in Algorithm \ref{alg:2}, we can deduce that for all $k \in [K], h \in [H]$, $\left\|w_{k, h} \bphi_{V_{k, h + 1}}(s_h^k, a_h^k) \right\|_{\hat \bSigma_{k, h, \ell_{k, h}}^{-1}} \le 2^{-\ell_{k, h}}$. According to Theorem \ref{thm:bernstein1}, for layer $\ell \in [L]$, we have with probability at least $1 - \delta / L$, for all $k \in [K]$: \begin{align} 
            \|\hat \btheta_{k, \ell} \|_{\hat \bSigma_{k, 1, \ell}} \le 16 \cdot 2^{-\ell} \sqrt{\sum_{(i, h) \in \Psi_{k, \ell}} w_{i, h}^2 \sigma_{i, h}^2 \log(4k^2H^2 L / \delta)} + 6 \cdot 2^{-\ell} \log(4k^2H^2 L /\delta) + 2^{-\ell} \sqrt{\lambda}  B. \notag
        \end{align}
    After applying a union bound over $\ell \in [L]$, we complete the proof of Lemma \ref{lemma:event:sub}.
    \end{proof}
    
    \begin{lemma} \label{lemma:var'}
        Let $\cE_{\mathrm{var}'}$ be defined in \eqref{eq:event:var'}. We have $\PP(\cE_{\mathrm{var}'}) \ge 1 - 2\delta$. 
    \end{lemma}
    \begin{proof} 
        By the definition of $\VV$ and Definition \ref{assumption-linear}, we have \begin{align*}
            \EE[\epsilon_{k, h}^2 | \cF_{k, h}] = \sigma_{k, h}^2,\ \PP(\left|\epsilon_{k, h} \right| \le 1) = 1. 
        \end{align*}
        Equivalent as the proof of Lemma \ref{lemma:var}, we can prove that with probability at least $1 - 2\delta$, for all episode $k \ge 1$ and layer $\ell\in[L]$, \begin{align*} 
            \sum_{(i, h) \in \Psi_{k, \ell}} w_{i, h}^2 \left|\epsilon_{i, h}^2 - \sigma_{i, h}^2\right| \le \frac{1}{2} \sum_{(i, h) \in \Psi_{k, \ell}} w_{i, h}^2 \sigma_{i, h}^2 + \frac{7}{3} \log\left(4k^2H^2 L / \delta\right), 
        \end{align*}
        which completes the proof of Lemma \ref{lemma:var'}.
    \end{proof}
    \subsection{Proof of Optimism}
    \begin{lemma} \label{lemma:varest:mdp}
        Let $w_{k, h}$ be defined in Algorithm \ref{alg:2}. On the event $\cE_{\mathrm{c}}$ and $\cE_{\mathrm{var}'}$, for all $k \ge 1$, $\ell \in [L]$ such that $2^\ell \ge 64\sqrt{\log(4k^2H^2 L / \delta)}$, the following inequalities hold: \begin{align*} 
            \sum_{(i, h) \in \Psi_{k, \ell}} w_{i, h}^2 \sigma_{i, h}^2 &\le 8 \sum_{(i, h) \in \Psi_{k, \ell}} w_{i, h}^2 \left(V_{i, h + 1}(s_{h + 1}^i) - \big\la \hat \btheta_{k, \ell}, \bphi_{V_{i, h + 1}}(s_h^i, a_h^i)\big\ra\right)^2 \\&\quad + 8 \log(4k^2H^2 L /\delta) + 2^{-2\ell + 5} \cdot \lambda B^2, \\
            \sum_{(i, h) \in \Psi_{k, \ell}} w_{i, h}^2  \Big(V_{i, h + 1}(s_{h + 1}^i) & - \big\la \hat \btheta_{k, \ell}, \bphi_{V_{i, h + 1}}(s_h^i, a_h^i)\big\ra\Big)^2 \le\frac{3}{2} \sum_{(i, h) \in \Psi_{k, \ell}} w_{i, h}^2 \sigma_{i, h}^2 + 2^{-2\ell} \lambda B^2 \\& \quad + \frac{7}{3} \log\left(4k^2H^2 L / \delta\right).
        \end{align*}
    \end{lemma}
    \begin{proof} 
        Let $\ell$ be an arbitrary layer in $[L]$. 
        According to the definition of event $\cE_{\mathrm{var}'}$, we have \begin{align} 
            \sum_{(i, h) \in \Psi_{k, \ell}} w_{i, h}^2 \sigma_{i, h}^2 &\le 2\sum_{(i, h) \in \Psi_{k, \ell}} w_{i, h}^2 \epsilon_{i, h}^2 + \frac{14}{3} \log\left(4k^2H^2 L / \delta\right)\notag
            \\&\le 4\sum_{(i, h) \in \Psi_{k, \ell}} w_{i, h}^2 \left(V_{i, h + 1}(s_{h + 1}^i) - \big\la \hat \btheta_{k, \ell}, \bphi_{V_{i, h + 1}}(s_h^i, a_h^i)\big\ra\right)^2 \notag
            \\&\quad + 4\sum_{(i, h) \in \Psi_{k, \ell}} w_{i, h}^2 \left[\epsilon_{i, h} - \left(V_{i, h + 1}(s_{h + 1}^i) - \big\la \hat \btheta_{k, \ell}, \bphi_{V_{i, h + 1}}(s_h^i, a_h^i)\big\ra\right)\right]^2 \notag
            \\&\quad + \frac{14}{3} \log\left(4k^2H^2 L / \delta\right), \label{eq:varest:mdp:0}
        \end{align}
        where the last inequality holds due to the fact $(a + b)^2 \le 2a^2 + 2b^2$. Then we consider the second term and we have \begin{align} 
            &\sum_{(i, h) \in \Psi_{k, \ell}} w_{i, h}^2 \left[\epsilon_{i, h} - \left(V_{i, h + 1}(s_{h + 1}^i) - \big\la \hat \btheta_{k, \ell}, \bphi_{V_{i, h + 1}}(s_h^i, a_h^i)\big\ra\right)\right]^2 \notag
            \\&\quad = \sum_{(i, h) \in \Psi_{k, \ell}} w_{i, h}^2 \left(\big\la \btheta^* - \hat \btheta_{k, \ell}, \bphi_{V_{i, h + 1}}(s_h^i, a_h^i) \big\ra\right)^2 \notag
            \\&\quad = \sum_{(i, h) \in \Psi_{k, \ell}} w_{i, h}^2 \left(\btheta^* - \hat \btheta_{k, \ell}\right)^\top \bphi_{V_{i, h + 1}}(s_h^i, a_h^i) \bphi_{V_{i, h + 1}}(s_h^i, a_h^i)^\top \left(\btheta^* - \hat \btheta_{k, \ell}\right) \notag
            \\&\quad \le \left\|\btheta^* - \hat \btheta_{k, \ell}\right\|_{\hat\bSigma_{k, \ell}}^2 \notag
            \\&\quad \le \left(16 \cdot 2^{-\ell} \sqrt{\sum_{(i, h) \in \Psi_{k, \ell}} w_{i, h}^2 \sigma_{i, h}^2 \log(4k^2H^2 L / \delta)} + 6 \cdot 2^{-\ell} \log(4k^2H^2 L /\delta) + 2^{-\ell} \sqrt{\lambda} \cdot B \right)^2, \label{eq:varest:mdp:1}
        \end{align}
        where the inequality holds due to $\hat \bSigma_{k, \ell}\succeq \bphi_{V_{i, h + 1}}(s_h^i, a_h^i) \bphi_{V_{i, h + 1}}(s_h^i, a_h^i)^\top $ and weight $w_{i,h}\leq 1$, the last equality follows from the definition of $\cE_{\mathrm{c}}$. In addition, from \eqref{eq:varest:mdp:1}, when $2^\ell \ge 64\sqrt{\log(4k^2H^2 L / \delta)}$, \begin{align} 
            &\sum_{(i, h) \in \Psi_{k, \ell}} w_{i, h}^2\left[\epsilon_{i, h} - \left(V_{i, h + 1}(s_{h + 1}^i) - \big\la \hat \btheta_{k, \ell}, \bphi_{V_{i, h + 1}}(s_h^i, a_h^i)\big\ra\right)\right]^2 \notag \\&\le \frac{1}{8} \sum_{(i, h) \in \Psi_{k, \ell}} w_{i, h}^2 \sigma_{i, h}^2 + 2\left(6 \cdot 2^{-\ell} \log(4k^2H^2 L /\delta) + 2^{-\ell} \sqrt{\lambda} \cdot B\right)^2 \notag
            \\&\le \frac{1}{8} \sum_{(i, h) \in \Psi_{k, \ell}} w_{i, h}^2 \sigma_{i, h}^2 + \log(4k^2H^2 L /\delta) + 2^{-2\ell + 2} \cdot \lambda B^2, \label{eq:varest:mdp:2}
        \end{align}
    where the first inequality and the second inequality hold due to the fact that $(a+b)^2\leq 2 a^2+2 b^2$. Substituting \eqref{eq:varest:mdp:2} into \eqref{eq:varest:mdp:0}, we have \begin{align*} 
            \sum_{(i, h) \in \Psi_{k, \ell}} w_{i, h}^2 \sigma_{i, h}^2 &\le 4\sum_{(i, h) \in \Psi_{k, \ell}} w_{i, h}^2 \left(V_{i, h + 1}(s_{h + 1}^i) - \big\la \hat \btheta_{k, \ell}, \bphi_{V_{i, h + 1}}(s_h^i, a_h^i)\big\ra\right)^2
            \\&\quad + \frac{1}{2} \sum_{(i, h) \in \Psi_{k, \ell}} w_{i, h}^2 \sigma_{i, h}^2 + 4 \log(4k^2H^2 L /\delta) + 2^{-2\ell + 4} \cdot \lambda B^2
            \\&\le 8 \sum_{(i, h) \in \Psi_{k, \ell}} w_{i, h}^2 \left(V_{i, h + 1}(s_{h + 1}^i) - \big\la \hat \btheta_{k, \ell}, \bphi_{V_{i, h + 1}}(s_h^i, a_h^i)\big\ra\right)^2 \\&\quad + 8 \log(4k^2H^2 L /\delta) + 2^{-2\ell + 5} \cdot \lambda B^2, 
        \end{align*}
        where the last inequality holds due to the fact that $x\leq x/2 +y$ implies $x\leq 2y$.
        Thus, we complete the proof of the first inequality in this lemma. 
    
        Note that $\hat\btheta_{k, \ell}$ is the minimizer of \begin{align*} 
            \hat\btheta_{k, \ell} \leftarrow \arg\min_{\btheta \in \RR^d}\sum_{(i, h) \in \Psi_{k, \ell}} w_{i, h}^2 \left(V_{i, h + 1}(s_{h + 1}^i) - \la \btheta, \bphi_{V_{i, h + 1}}(s_h^i, a_h^i)\ra\right)^2 + 2^{-2\ell} \lambda \|\btheta\|_2^2, 
        \end{align*} and we have \begin{align*} 
            &\sum_{(i, h) \in \Psi_{k, \ell}} w_{i, h}^2 \left(V_{i, h + 1}(s_{h + 1}^i) - \big\la \hat \btheta_{k, \ell}, \bphi_{V_{i, h + 1}}(s_h^i, a_h^i)\big\ra\right)^2 \\
            &\leq \sum_{(i, h) \in \Psi_{k, \ell}} w_{i, h}^2 \left(V_{i, h + 1}(s_{h + 1}^i) - \big\la \btheta^*, \bphi_{V_{i, h + 1}}(s_h^i, a_h^i)\big\ra\right)^2+2^{-2\ell}\lambda \|\btheta^*\|_2^2 \\
            &\le \sum_{(i, h) \in \Psi_{k, \ell}} w_{i, h}^2 \epsilon_{i, h}^2 + 2^{-2\ell} \lambda B^2
            \\&\le \frac{3}{2} \sum_{(i, h) \in \Psi_{k, \ell}} w_{i, h}^2 \sigma_{i, h}^2 + 2^{-2\ell} \lambda B^2 + \frac{7}{3} \log\left(4k^2H^2 L / \delta\right), 
        \end{align*}
        where the first inequality holds due to the definition of $\hat\btheta_{k, \ell}$, the second inequality holds due to $\|\btheta^*\|\leq B$ and the last inequality follows from the definition of $\cE_{\mathrm{var}'}$. Therefore, we complete the proof of Lemma \ref{lemma:varest:mdp}. 
    \end{proof}
    \begin{lemma}\label{lemma:upper:bellman}
        Let value function $\qvalue_{k,h}, \vvalue_{k,h}$ and confidence radius $\hat{\beta}_{k,\ell}$ be defined in Algorithm \ref{alg:2}. Suppose that $\lambda = 1 / B^2$ in Algorithm \ref{alg:2}. 
        Then, on the event $\cE_{\mathrm{var}'} \cap \cE_{\mathrm{c}}$, 
        for any $(k, h)\in  [K] \times [H]$, we have 
        $[\PP\vvalue_{k, h + 1}](s_h^k, a_h^k) \leq V_{k,h}(s_h^k)$.
    \end{lemma}
    \begin{proof} 
        From the definition of event $\cE_{\mathrm{var}'}$  and Lemma \ref{lemma:varest:mdp}, we can deduce that \begin{align*} 
            \hat \beta_{k, \ell} \ge 16 \cdot 2^{-\ell} \sqrt{\sum_{i = 1}^k \sigma_i^2 \log(4k^2H^2 L / \delta)} + 6 \cdot 2^{-\ell} \log(4k^2H^2 L /\delta) + 2^{-\ell} \sqrt{\lambda} \cdot B. 
        \end{align*} 
        Therefore, by the definition of $\cE_{\mathrm{c}}$, we have \begin{align} 
            \forall k \ge 1, \ell \in [L], \quad \|\hat \btheta_{k, \ell} - \btheta^* \|_{\hat \bSigma_{k, 1, \ell}} \le \hat \beta_{k, \ell}. \label{eq:upper:0}
        \end{align}
        According to Algorithm \ref{alg:2}, we have
        \begin{align*} 
            V_{k, h}(s_h^k) &= \min \left\{1, \min_{\ell \in [L]} \left\{r(s_h^k, a_h^k) + \left\la \hat \btheta_{k, \ell}, \bphi_{V_{k, h + 1}}(s_h^k, a_h^k) \right\ra + \hat \beta_{k, \ell} \left\|\bphi_{V_{k, h + 1}}(s_h^k, a_h^k)\right\|_{\hat \bSigma_{k, 1, \ell}^{-1}}\right\}\right\} 
            \\&\ge \min \left\{1, \min_{\ell \in [L]} \left\{\left\la \hat \btheta_{k, \ell}, \bphi_{V_{k, h + 1}}(s_h^k, a_h^k) \right\ra + \hat \beta_{k, \ell} \left\|\bphi_{V_{k, h + 1}}(s_h^k, a_h^k)\right\|_{\hat \bSigma_{k, 1, \ell}^{-1}}\right\}\right\} 
            \\&\ge \min \left\{1, \min_{\ell \in [L]} \left\{\left\la \btheta^*, \bphi_{V_{k, h + 1}}(s_h^k, a_h^k) \right\ra\right\}\right\} 
            \\&= [\PP\vvalue_{k, h + 1}](s_h^k, a_h^k), 
        \end{align*}
        where the first inequality holds due to $r(s_h^k, a_h^k) >0$, the second one follows from \eqref{eq:upper:0}, the last equality holds due to the definition of linear mixture MDPs and the fact that $V_{k, h + 1}(s) \le 1$ for all $s \in \cS$. Thus, we complete the proof of Lemma \ref{lemma:upper:bellman}.
    \end{proof}
    \begin{lemma}\label{lemma:upper:finite}
        Let value function $\qvalue_{k,h}, \vvalue_{k,h}$ and confidence radius $\hat{\beta}_{k,\ell}$ be defined in Algorithm \ref{alg:2}. Suppose that $\lambda = 1 / B^2$ in Algorithm \ref{alg:2}. 
        Then, on the event $\cE_{\mathrm{var}'} \cap \cE_{\mathrm{c}}$, 
        for any $(s,a, k, h)\in \cS \times \cA \times [K] \times [H]$, we have 
        $\qvalue_h^*(s,a) \leq \qvalue_{k,h}(s,a)$ and $\vvalue_h^*(s) \leq \vvalue_{k,h}(s)$.
    \end{lemma}
    \begin{proof} 
        From the definition of event $\cE_{\mathrm{var}'}$  and Lemma \ref{lemma:varest:mdp}, we can deduce that \begin{align*} 
            \hat \beta_{k, \ell} \ge 16 \cdot 2^{-\ell} \sqrt{\sum_{i = 1}^k \sigma_i^2 \log(4k^2H^2 L / \delta)} + 6 \cdot 2^{-\ell} \log(4k^2H^2 L /\delta) + 2^{-\ell} \sqrt{\lambda} \cdot B. 
        \end{align*} 
        Therefore, by the definition of $\cE_{\mathrm{c}}$, we have \begin{align} 
            \forall k \ge 1, \ell \in [L], \quad \|\hat \btheta_{k, \ell} - \btheta^* \|_{\hat \bSigma_{k, 1, \ell}} \le \hat \beta_{k, \ell}. \label{eq:upper:1}
        \end{align}
        Consider an arbitrary episode $k \in [K]$ in the remaining proof. If for some stage $h> 1$, the following inequalities $\qvalue_h^*(s,a) \leq \qvalue_{k,h}(s,a)$, $\vvalue_h^*(s) \leq \vvalue_{k,h}(s)$ hold for all $(s, a) \in \cS \times \cA$, then for any $(s, a) \in \cS \times \cA, \ell \in [L]$ and stage $h-1$, we have
    \begin{align*} 
            Q_{h - 1}^*(s, a) &= r(s, a) + \big\la \btheta^*, \bphi_{V_h^*} (s, a) \big\ra
            \\&\le r(s, a) + \big\la \btheta^*, \bphi_{V_{k, h}} (s, a) \big\ra 
            \\&\le r(s, a) + \big\la \hat \btheta_{k, \ell}, \bphi_{V_{k, h}} (s, a) \big\ra + \big\|\hat \btheta_{k, \ell} - \btheta^*\big\|_{\hat \bSigma_{k, 1, \ell}} \big\|\bphi_{V_{k, h}} (s, a)\big\|_{\hat \bSigma_{k, 1, \ell}^{-1}}
            \\&\le r(s, a) + \big\la \hat \btheta_{k, \ell}, \bphi_{V_{k, h}} (s, a) \big\ra + \hat \beta_{k, \ell} \big\|\bphi_{V_{k, h}} (s, a)\big\|_{\hat \bSigma_{k, 1, \ell}^{-1}}, 
        \end{align*}
        where the first inequality holds by our assumption that $\vvalue_h^*(s) \leq \vvalue_{k,h}(s)$, the second inequality holds due to Cauchy-Schwarz inequality and the last inequality follows from \eqref{eq:upper:1}. By the arbitrariness of layer $\ell$, we have $Q_{h - 1}^*(s, a) \le Q_{k, h - 1}(s, a)$ holds for all state-action pair $(s, a)$, which indicates that $\vvalue_{h - 1}^*(s) \leq \vvalue_{k,h - 1}(s)$ holds for all $s \in \cS$. 
        Since $0 = V_{H + 1}^*(\cdot) \le V_{k, H + 1}(\cdot)$ holds trivially for stage $H+1$, we complete the proof of Lemma \ref{lemma:upper:finite} by induction. 
    \end{proof}
    \subsection{Sum of Bellman Errors}  
    \begin{lemma} \label{lemma:sumbell}
        Let $\hat \beta_{k, \ell}$, $V_{k, h}$, $\bphi_{V_{k, h + 1}}$ be defined in Algorithm \ref{alg:2} and set $\lambda = 1 / B^2, \alpha = 1 / (KH)^{3 / 2}$. Then on the event $\cE_{\mathrm{var}'} \cap \cE_{\mathrm{c}} $, we have \begin{align*} 
            \sum_{k=1}^K \sum_{h=1}^H I_h^k\max\Big\{\big[V_{k,h}(s_h^k) - r(s_h^k, a_h^k) - [\PP V_{k,h+1}](s_h^k, a_h^k)\big], 0\Big\} \le \tilde{O}\left(d \sqrt{\sum_{k = 1}^K \sum_{h = 1}^H  \sigma_{k, h}^2} + d\right). 
        \end{align*}
    \end{lemma}
    
    \begin{proof} 
        For simplicity, let $\ell^*$ be the smallest $\ell$ in $L$ such that $2^{\ell} \ge 64\sqrt{\log\left(4K^2H^2L /\delta\right)}$. 
        According to Algorithm \ref{alg:2}, we have
        \begin{align*} 
            V_{k, h}(s_h^k) = \min \left\{1, \min_{\ell \in [L]} \left\{r(s_h^k, a_h^k) + \left\la \hat \btheta_{k, \ell}, \bphi_{V_{k, h + 1}}(s_h^k, a_h^k) \right\ra + \hat \beta_{k, \ell} \left\|\bphi_{V_{k, h + 1}}(s_h^k, a_h^k)\right\|_{\hat \bSigma_{k, 1, \ell}^{-1}}\right\}\right\}. 
        \end{align*}
        Therefore, we have 
        \begin{align} 
            &\sum_{k=1}^K \sum_{h=1}^H I_h^k\max\Big\{\big[V_{k,h}(s_h^k) - r(s_h^k, a_h^k) - [\PP V_{k,h+1}](s_h^k, a_h^k)\big], 0\Big\}
            \\&\le \sum_{k = 1}^K \sum_{h = 1}^H I_h^k  \bigg[\min_{\ell \in [L]} \left\{\left\la \hat \btheta_{k, \ell} - \btheta^*, \bphi_{V_{k, h + 1}}(s_h^k, a_h^k) \right\ra + \hat \beta_{k, \ell} \left\|\bphi_{V_{k, h + 1}}(s_h^k, a_h^k)\right\|_{\hat \bSigma_{k, 1, \ell}^{-1}}\right\}\bigg]_{[0, 2]}\notag
            \\&\le \sum_{k = 1}^K \sum_{h= 1}^H I_h^k \min\left\{2, \min_{\ell \in [L]} \left\{2 \hat \beta_{k, \ell} \left\|\bphi_{V_{k, h + 1}}(s_h^k, a_h^k)\right\|_{\hat \bSigma_{k, 1, \ell}^{-1}}\right\}\right\} \notag
            \\&\le \sum_{\ell = \ell^* + 1}^{L + 1} \sum_{(k, h) \in \Psi_{K + 1, \ell}} I_h^k \min\left\{2, 2 \hat \beta_{k, \ell - 1} \left\|\bphi_{V_{k, h + 1}}(s_h^k, a_h^k)\right\|_{\hat \bSigma_{k, 1, \ell - 1}^{-1}}\right\} + 2 \sum_{\ell = 1}^{\ell^*} |\Psi_{K + 1, \ell}|, \label{eq:ellnorm0}
        \end{align}
    where the first inequality holds due to the definition of value function $V_{k,h}(s_h^k)$, the second inequality holds due to Cauchy-Schwarz inequality with event $\cE_c$ and the last inequality holds since indicator function $I_h^k\leq 1$. By the definition of indicator function $I_h^k$ and Lemma \ref{lemma:det}, we further have \begin{align} 
            \left\|\bphi_{V_{k, h + 1}}(s_h^k, a_h^k)\right\|_{\hat \bSigma_{k, 1, \ell_{k, h} - 1}^{-1}} &\le 2\left\|\bphi_{V_{k, h + 1}}(s_h^k, a_h^k)\right\|_{\hat \bSigma_{k, h, \ell_{k, h} - 1}^{-1}} \notag
            \\&\le 2 \cdot 2^{-\ell_{k, h} + 1}, \label{eq:ellnorm1}
        \end{align}
        where the last inequality follows from the definition of $\ell_{k, h}$ in Algorithm \ref{alg:2}. 
        Substituting \eqref{eq:ellnorm1} into \eqref{eq:ellnorm0}, we have \begin{align*} 
            &\sum_{k=1}^K \sum_{h=1}^H I_h^k\max\Big\{\big[V_{k,h}(s_h^k) - r(s_h^k, a_h^k) - [\PP V_{k,h+1}](s_h^k, a_h^k)\big], 0\Big\} \notag
            \\&\le \sum_{\ell = \ell^* + 1}^{L + 1} |\Psi_{K + 1, \ell}|\cdot \tilde{O}\left(2^{-\ell} \cdot 2^{-\ell} \sqrt{\sum_{k = 1}^K \sum_{h = 1}^H \sigma_{k, h}^2} + 2^{-2\ell}\right) + 2\sum_{\ell = 1}^{\ell^*} |\Psi_{K + 1, \ell}|
            \\&\le \tilde{O}\left(d\sqrt{\sum_{k = 1}^K \sum_{h = 1}^H \sigma_{k, h}^2} + d\right),
        \end{align*}
        where the first inequality follows from the definition of $\hat\beta_{k, \ell}$ in Algorithm \ref{alg:2} and Lemma \ref{lemma:varest:mdp}, the last inequality holds due to Lemma \ref{lemma:card} and the definition of $L$. Thus, we complete the proof of Lemma \ref{lemma:sumbell}.
    \end{proof}
    
    \subsection{Quantities in MDP} 
    
    In this subsection, we define the following quantities:
    We use $\check V_{k, h}(s)$ to denote the estimation error between the optimistic value function and the actually optimal value function, and use $\tilde V_{k, h}(s)$ to denote the sub-optimality gap of policy $\pi_k$ at stage $h$:
    \begin{align}
         & \check V_{k, h}(s) = V_{k, h}(s) - V_h^*(s), \quad \forall s \in \cS, (k, h) \in [K]\times[H] \\
        & \tilde V_{k, h}(s) = V_h^*(s) - V_h^{\pi_k}(s), \quad \forall s \in \cS, (k, h) \in [K]\times[H]
    \end{align}
    We use $Q_0,S_m,\check{S}_m,\tilde{S}_m$ to represent the total variances of optimal value function $V_{h+1}^*$ and $2^m$-th order value functions ($\vvalue_{k, h+1}^{2^m}, \check\vvalue_{k, h+1}^{2^m} , \tilde\vvalue_{k, h+1}^{2^m}$):
    \begin{align}
         & S_m = \sum_{k=1}^K\sum_{h=1}^H[\VV\vvalue_{k, h+1}^{2^m}](s_h^k, a_h^k)\label{eq:def:sm}, \\
        & \check S_m = \sum_{k=1}^K\sum_{h=1}^H[\VV\check\vvalue_{k, h+1}^{2^m}](s_h^k, a_h^k)\label{eq:def:checksm}, \\
        & \tilde S_m = \sum_{k=1}^K\sum_{h=1}^H[\VV\tilde\vvalue_{k, h+1}^{2^m}](s_h^k, a_h^k)\label{eq:def:tildesm}, \\
        & Q_0 = \sum_{k = 1}^K \sum_{h = 1}^H [\VV V_{h + 1}^*](s_h^k, a_h^k), \label{eq:def:q0}
    \end{align}
    where $Q_0$ is introduced as a shorthand for $\Var_K^*$ for simplicity. 
    In addition, for  $2^m$-th order value functions ($\vvalue_{k, h+1}^{2^m}, \check\vvalue_{k, h+1}^{2^m} , \tilde\vvalue_{k, h+1}^{2^m}$) and optimistic value function $\vvalue_{k,h}$, we denote the summation of stochastic transition noise as follows:
    \begin{align}
        & A_m = \left|\sum_{k=1}^K \sum_{h=1}^H [[\PP V_{k,h+1}^{2^m}](s_h^k, a_h^k) - V_{k,h+1}^{2^m}(s_{h+1}^k)]\right|, \label{eq:def:am} \\
        & \check A_m = \left|\sum_{k = 1}^K \sum_{h=1}^H [[\PP \check V_{k,h+1}^{2^m}](s_h^k, a_h^k) - \check V_{k,h+1}^{2^m}(s_{h+1}^k)]\right|, \label{eq:def:checkam} \\
        & \tilde A_m = \left|\sum_{k = 1}^K \sum_{h=1}^H [[\PP \tilde V_{k,h+1}^{2^m}](s_h^k, a_h^k) - \tilde V_{k,h+1}^{2^m}(s_{h+1}^k)]\right|, \label{eq:def:tildeam} \\
        & R_0 = \sum_{k=1}^K \sum_{h=1}^H I_h^k \max\Big\{\big[V_{k,h}(s_h^k) - r(s_h^k, a_h^k) - [\PP V_{k,h+1}](s_h^k, a_h^k)\big], 0\Big\}. \label{eq:def:r0}
    \end{align}
    Finally, we use the quantity $G$ to denote the number of episodes when the determinant of covariance matrix grows sharply:
    \begin{align} 
    G = \sum_{k=1}^K(1-I_H^k), \label{eq:def:g}
    \end{align} where indicator function $I_h^k$ is defined in \eqref{eq:def:ind}. 
    For the above quantities, we only consider $m \in [M]$ where $M := \lceil \log_2 2HK \rceil$. Now, we introduce the following lemmas to build the connection between these quantities. 
    
    To construct the connections and upper bounds of the quantities above, our proof in this subsection follows the previous approaches proposed by \citet{zhang2021variance} and \citet{zhou2022computationally}, but with a more fine-grained analysis to remove explicit $K$-dependence. 
     
    \begin{lemma} \label{lemma:rec:1}
        Let $\check S_m$, $A_m$, $R_0$, $G$ be defined in \eqref{eq:def:checksm}, \eqref{eq:def:am}, \eqref{eq:def:r0}, \eqref{eq:def:g}. On the event $\cE_{\mathrm{var}'} \cap \cE_{\mathrm{c}}$, we have the following inequalities for all $m \in [M]$: 
        \begin{align*} 
            \check S_m \le \check A_{m + 1} + G + 2^{m + 1}\cdot (R_0 + G + A_0). 
        \end{align*}
    \end{lemma}
    \begin{proof}
        Based on the definition of $\check S_m$, we compute 
        \begin{small}
        \begin{align} 
            \check S_m & = \sum_{k = 1}^K \sum_{h = 1}^H [\VV \check V_{k, h + 1}^{2^m}](s_h^k, a_h^k) \notag 
            \\&= \sum_{k = 1}^K \sum_{h = 1}^H \left[[\PP \check V_{k, h + 1}^{2^{m + 1}}](s_h^k, a_h^k) - \left([\PP \check V_{k, h + 1}^{2^m}](s_h^k, a_h^k)\right)^2\right] \notag
            \\&= \sum_{k = 1}^K \sum_{h = 1}^H \left[[\PP \check V_{k, h + 1}^{2^{m + 1}}](s_h^k, a_h^k) - \check V_{k, h + 1}^{2^{m + 1}}(s_{h + 1}^k)\right]
             +  \sum_{k = 1}^K \sum_{h = 1}^H \left[\check V_{k, h}^{2^{m + 1}}(s_{h}^k) - \left([\PP \check V_{k, h + 1}^{2^m}](s_h^k, a_h^k)\right)^2\right]. \label{eq:rec:0}
        \end{align}
        \end{small}
        For the second term, it can be further upper bounded by \begin{align} 
            &\sum_{k = 1}^K \sum_{h = 1}^H \left[\check V_{k, h}^{2^{m + 1}}(s_{h}^k) - \left([\PP \check V_{k, h + 1}^{2^m}](s_h^k, a_h^k)\right)^2\right] \notag
            \\&\le \sum_{k = 1}^K \sum_{h = 1}^H \left[\check V_{k, h}^{2^{m + 1}}(s_{h}^k) - \left([\PP \check V_{k, h + 1}](s_h^k, a_h^k)\right)^{2^{m + 1}}\right] \notag
            \\&\le 2^{m + 1} \sum_{k = 1}^K \sum_{h = 1}^H \max\left\{\check V_{k, h}(s_h^k) - [\PP \check V_{k, h + 1}](s_h^K, a_h^k), 0\right\} \notag
            \\&\le 2^{m + 1} \sum_{k = 1}^K \sum_{h = 1}^H I_h^k \max\Big\{\big[V_{k,h}(s_h^k) - r(s_h^k, a_h^k) - [\PP V_{k,h+1}](s_h^k, a_h^k)\big], 0\Big\} \notag
            \\&\quad + 2^{m + 1} \sum_{k = 1}^K (1 - I_H^k) \sum_{h = 1}^H \max\Big\{\big[V_{k,h}(s_h^k) - r(s_h^k, a_h^k) - [\PP V_{k,h+1}](s_h^k, a_h^k)\big], 0\Big\} \notag
            \\&\le 2^{m + 1} R_0 
             + 2^{m + 1} \sum_{k = 1}^K (1 - I_H^k) \sum_{h = 1}^H  \big[V_{k,h + 1}(s_{h + 1}^k) - [\PP V_{k,h+1}](s_h^k, a_h^k)\big] \notag
             \\&\quad + 2^{m + 1} \sum_{k = 1}^K (1 - I_H^k) \sum_{h = 1}^H \left(V_{k, h}(s_h^k) - V_{k, h + 1}(s_{h + 1}^k)\right) \notag
            \\&\le 2^{m + 1} \cdot (R_0 + G + A_0), \label{eq:rec:1}
        \end{align}
        where the first inequality holds due to 
        \begin{align*}
            \left([\PP \check V_{k, h + 1}^{2^m}](s_h^k, a_h^k)\right)^2 \ge \left([\PP \check V_{k, h + 1}^{2^{m - 1}}](s_h^k, a_h^k)\right)^4 \ge \cdots \ge \left([\PP \check V_{k, h + 1}](s_h^k, a_h^k)\right)^{2^{m + 1}},
        \end{align*}
         the second inequality follows from the fact that $a^x - b^x \le x \max\{a - b, 0\}$ for $a, b \in [0, 1]$ and $x\ge 1$, the third inequality follows from the monotonicity of indicator function $I_h^k$ and the definition of function $\check{V}_{k,h}$, the fourth holds since $r(s_h^k,a_h^k)\ge 0$ and
         the last inequality holds due to Lemma \ref{lemma:upper:bellman}. 
         
        Substituting \eqref{eq:rec:1} into \eqref{eq:rec:0}, we have \begin{align*} 
            \check S_m \le \check A_{m + 1} + G + 2^{m + 1}\cdot (R_0 + G + A_0). 
        \end{align*}
        Thus, we complete the proof of Lemma \ref{lemma:rec:1}.
    \end{proof}
    
    \begin{lemma} \label{lemma:rec:2}
        Let $A_m$, $Q_0$, $\check S_m$ be defined in \eqref{eq:def:am}, \eqref{eq:def:q0}, \eqref{eq:def:checksm}. Then with probability at least $1 - 2\delta$, we have \begin{align*} 
            A_0 \le 2\sqrt{(Q_0 + \check S_0) \log(1 / \delta)} + (2 / 3)\cdot \log(1 / \delta). 
        \end{align*}
        For simplicity, we denote the corresponding event by $\cE_{\mathrm{r}_1}$. 
    \end{lemma}
    \begin{proof} 
        Applying the Freedman's inequality in Lemma \ref{lemma:freedman}, we have with probability at least $1 - 2\delta$, \begin{align}
            A_0 = \Bigg|\sum_{k=1}^K \sum_{h=1}^H \big[[\PP V_{k,h+1}](s_h^k, a_h^k) - V_{k,h+1}(s_{h+1}^k)\big] \Bigg|\le \sqrt{2\sum_{k = 1}^K \sum_{h = 1}^H \sigma_{k, h}^2 \log(1 / \delta)} + \frac{2}{3} \log(1 / \delta), \label{eq:rec:2}
        \end{align}
        where the variance $\sigma_{k, h}$ is defined in \eqref{eq:def:sigma}. For variance $\sigma_{k, h}$, we 
     further have \begin{align} 
            \sigma_{k, h}^2 = [\VV V_{k, h + 1}] (s_h^k, a_h^k) \le 2[\VV V_{h + 1}^*] (s_h^k, a_h^k) + 2[\VV \check V_{k, h + 1}] (s_h^k, a_h^k), \label{eq:rec:3}
        \end{align}
        where the inequality holds due to the fact that $\Var(x+y)\leq 2\Var(x)+2\Var(y)$.
    Substituting \eqref{eq:rec:3} into \eqref{eq:rec:2}, we  complete the proof of Lemma \ref{lemma:rec:2}. 
    \end{proof}
    \begin{lemma} \label{lemma:rec:3}
        Let $\check A_m$, $\check S_m$ be defined in \eqref{eq:def:am}, \eqref{eq:def:checksm}. 
        With probability at least $1 - 2(M + 1) \delta$, for all $m \in [M] \cup \{0\}$, we have
        \begin{align*} 
            \check A_m \le \sqrt{2 \check S_m \log(1 / \delta)}+ \frac{4}{3} \cdot \log(1 / \delta). 
        \end{align*}
        We denote the corresponding event by $\cE_{\mathrm{r}_2}$.
    \end{lemma}
    \begin{proof} 
        Note that for all state $s$, $\check V_{k, h}(s)=V_{k,h}(s)-V_h^{*}(s) \in [-1, 1]$. 
        Applying Freedman's inequality in Lemma \ref{lemma:freedman}, we have with probability at least $ 1- 2\delta$: \begin{align*}
            \check A_m \le \sqrt{2 \check S_m \log(1 / \delta)} + \frac{4}{3} \cdot \log(1 / \delta),
        \end{align*}
        for each $m \in [M]$. Thus, we complete the proof of Lemma \ref{lemma:rec:3} by using a union bound over $m \in [M]$. 
    \end{proof}
    
    \begin{lemma} \label{lemma:rec:4}
        Let $A_m$, $Q_0$, $\check A_m$, $R_0$, $G$ be defined in \eqref{eq:def:am}, \eqref{eq:def:q0}, \eqref{eq:def:checkam}, \eqref{eq:def:r0}, \eqref{eq:def:g}. On the event $\cE_{\mathrm{r}_1} \cap \cE_{\mathrm{var}'} \cap \cE_{\mathrm{c}}$, we have \begin{align*} 
            A_0 \le 4\sqrt{\big(Q_0 + \check A_1 + G + 2(R_0 + G)\big)\log(1 / \delta)} + 10\cdot \log(1 / \delta). 
        \end{align*}
    \end{lemma}
    
    \begin{proof} 
        According to Lemma \ref{lemma:rec:1} and Lemma \ref{lemma:rec:2}, we have \begin{align*} 
            A_0 &\le 2\sqrt{(Q_0 + \check S_0) \log(1 / \delta)} + (2 / 3)\cdot \log(1 / \delta)
            \\&\le 2\sqrt{\big(Q_0 + \check A_1 + G + 2(R_0 + G + A_0)\big)\log(1 / \delta)} + (2 / 3)\cdot \log(1 / \delta)
            \\&\le 2\sqrt{\big(Q_0 + \check A_1 + G + 2(R_0 + G)\big)\log(1 / \delta)} + (2 / 3)\cdot \log(1 / \delta) + 2\sqrt{2A_0\log(1 / \delta)} 
            \\&\le 4\sqrt{\big(Q_0 + \check A_1 + G + 2(R_0 + G)\big)\log(1 / \delta)} + 10\cdot \log(1 / \delta), 
        \end{align*}
        where the first inequality holds due to Lemma \ref{lemma:rec:1}, the second inequality holds due to \ref{lemma:rec:2} and
        the last inequality holds due to $x \le a\sqrt{x} + b \Rightarrow x \le a^2 + 2b$. 
        Thus, we complete the proof of Lemma \ref{lemma:rec:4}.
    \end{proof}
    
    \begin{lemma} \label{lemma:rec:6}
        Let $A_m$, $G$, $R_0$ be defined in \eqref{eq:def:am}, \eqref{eq:def:g}, \eqref{eq:def:r0}. On the event $\cE_{\mathrm{r}_2} \cap \cE_{\mathrm{var}'} \cap \cE_{\mathrm{c}}$, we have \begin{align*} 
            \check A_1 \le 4\sqrt{(R_0 + 2G + A_0) \log (1 / \delta)} + 11 \log(1 / \delta). 
        \end{align*}
    \end{lemma}
    
    \begin{proof} 
        By the definition of $\cE_{\mathrm{r}_2}$ in Lemma \ref{lemma:rec:3}, we have \begin{align} 
            \check A_m \le \sqrt{2 \check S_m \log(1 / \delta)}+ \frac{4}{3} \cdot \log(1 / \delta). \label{eq:rec:5}
        \end{align}
    Substituting the bound of $\check S_m$ in Lemma \ref{lemma:rec:1} into \eqref{eq:rec:5}, \begin{align*} 
            \check A_m \le \sqrt{2 \left(\check A_{m + 1} + 2^{m + 1}\cdot (R_0 + 2G + A_0)\right) \log(1 / \delta)}+ \frac{4}{3} \cdot \log(1 / \delta). 
        \end{align*}
        Applying Lemma \ref{lemma:rec:5}, we have \begin{align*} 
            \check A_1 &\le \max\left\{11 \log (1 / \delta), 4\sqrt{(R_0 + 2G + A_0) \log (1 / \delta)} + 2 \log(1 / \delta)\right\}
            \\&\le  4\sqrt{(R_0 + 2G + A_0) \log (1 / \delta)} + 11 \log(1 / \delta). 
        \end{align*}
        Thus, we complete the proof of Lemma \ref{lemma:rec:6}.
    \end{proof}
    
    \begin{lemma} \label{lemma:rec:7}
         Let $A_m$, $G$, $R_0$, $Q_0$ be defined in \eqref{eq:def:am}, \eqref{eq:def:g}, \eqref{eq:def:r0}, \eqref{eq:def:q0}. On the event $\cE_{\mathrm{r}_1}\cap \cE_{\mathrm{r}_2} \cap \cE_{\mathrm{var}'} \cap \cE_{\mathrm{c}}$, we have \begin{align*} 
            A_0 \le 132 \log(1 / \delta) + 28 \sqrt{R_0 \log(1 / \delta)} + 40 \sqrt{G \log(1 / \delta)} + 8\sqrt{Q_0 \log(1 / \delta)}. 
        \end{align*}
    \end{lemma}
    
    \begin{proof} 
        We compute \begin{align*} 
            A_0 &\le 4\sqrt{\big(Q_0 + G + 2(R_0 + G)\big)\log(1 / \delta)} + 10\cdot \log(1 / \delta) + 4\sqrt{\check A_1 \log(1 / \delta)}
            \\&\le 4\sqrt{\big(Q_0 + G + 2(R_0 + G)\big)\log(1 / \delta)} + 2\check A_1 + 12\log(1 / \delta)
            \\&\le  8\sqrt{(R_0 + 2G + A_0) \log (1 / \delta)} + 34 \log(1 / \delta) + 4\sqrt{(Q_0 + G + 2(R_0 + G))\log(1 / \delta)}
            \\&\le 132 \log(1 / \delta) + 28 \sqrt{R_0 \log(1 / \delta)} + 40 \sqrt{G \log(1 / \delta)} + 8\sqrt{Q_0 \log(1 / \delta)}, 
        \end{align*}
        where the first inequality follows from Lemma \ref{lemma:rec:4}, the second inequality holds due to the fact that $2ab\leq a^2+b^2$, the third inequality holds due to Lemma \ref{lemma:rec:6} and
        the last inequality holds due to the fact that $x \le a\sqrt{x} + b \Rightarrow x \le a^2 + 2b$. Thus, we complete the proof of Lemma \ref{lemma:rec:7}.
    \end{proof}
    
    \begin{lemma} \label{lemma:rec:8}
        Let $\tilde S_m$, $A_m$, $R_0$, $G$ be defined in \eqref{eq:def:tildesm}, \eqref{eq:def:am}, \eqref{eq:def:r0}, \eqref{eq:def:g}. On the event $\cE_{\mathrm{var}'} \cap \cE_{\mathrm{c}}$, we have the following inequalities for all $m \in [M]$: 
        \begin{align*} 
            \tilde S_m \le \tilde A_{m + 1} + G + 2^{m + 1}\cdot (R_0 + G + A_0). 
        \end{align*}
    \end{lemma}
    
    \begin{proof}
        Based on the definition of $\tilde S_m$, we compute 
        \begin{small}
        \begin{align} 
            \tilde S_m & = \sum_{k = 1}^K \sum_{h = 1}^H [\VV \tilde V_{k, h + 1}^{2^m}](s_h^k, a_h^k) \notag 
            \\&= \sum_{k = 1}^K \sum_{h = 1}^H \left[[\PP \tilde V_{k, h + 1}^{2^{m + 1}}](s_h^k, a_h^k) - \left([\PP \tilde V_{k, h + 1}^{2^m}](s_h^k, a_h^k)\right)^2\right] \notag
            \\&= \sum_{k = 1}^K \sum_{h = 1}^H \left[[\PP \tilde V_{k, h + 1}^{2^{m + 1}}](s_h^k, a_h^k) - \tilde V_{k, h + 1}^{2^{m + 1}}(s_{h + 1}^k)\right]
             +  \sum_{k = 1}^K \sum_{h = 1}^H \left[\tilde V_{k, h}^{2^{m + 1}}(s_{h}^k) - \left([\PP \tilde V_{k, h + 1}^{2^m}](s_h^k, a_h^k)\right)^2\right], \label{eq:rec:11}
        \end{align}
        \end{small}
        For the second term, we further have \begin{align} 
            &\sum_{k = 1}^K \sum_{h = 1}^H \left[\tilde V_{k, h}^{2^{m + 1}}(s_{h}^k) - \left([\PP \tilde V_{k, h + 1}^{2^m}](s_h^k, a_h^k)\right)^2\right] \notag
            \\&\le \sum_{k = 1}^K \sum_{h = 1}^H \left[\tilde V_{k, h}^{2^{m + 1}}(s_{h}^k) - \left([\PP \tilde V_{k, h + 1}](s_h^k, a_h^k)\right)^{2^{m + 1}}\right] \notag
            \\&\le 2^{m + 1} \sum_{k = 1}^K \sum_{h = 1}^H \max\left\{\tilde V_{k, h}(s_h^k) - [\PP \tilde V_{k, h + 1}](s_h^k, a_h^k), 0\right\} \notag
            \\&\le 2^{m + 1} \sum_{k = 1}^K \sum_{h = 1}^H I_h^k \big[V_{h}^*(s_h^k) - r(s_h^k, a_h^k) - [\PP V_{h+1}^*](s_h^k, a_h^k)\big] \notag
            \\&\quad + 2^{m + 1} \sum_{k = 1}^K (1 - I_H^k) \sum_{h = 1}^H  \big[V_{h}^*(s_h^k) - r(s_h^k, a_h^k) - [\PP V_{h+1}^*](s_h^k, a_h^k)\big] \notag
            \\&\le 2^{m + 1} R_0 +  2^{m + 1} \check A_0 
             + 2^{m + 1} \sum_{k = 1}^K (1 - I_H^k) \sum_{h = 1}^H  \big[V_{k,h + 1}(s_{h + 1}^k) - [\PP V_{k,h+1}](s_h^k, a_h^k)\big] \notag
             \\&\quad + 2^{m + 1} \sum_{k = 1}^K (1 - I_H^k) \sum_{h = 1}^H \left(V_{k, h}(s_h^k) - r(s_h^k, a_h^k) - V_{k, h + 1}(s_{h + 1}^k)\right) \notag
            \\&\le 2^{m + 1} \cdot (R_0 + G + A_0 + \check A_0), \label{eq:rec:12}
        \end{align}
        where the first inequality holds since 
        \begin{align*}
            \left([\PP \tilde V_{k, h + 1}^{2^m}](s_h^k, a_h^k)\right)^2 \ge \left([\PP \tilde V_{k, h + 1}^{2^{m - 1}}](s_h^k, a_h^k)\right)^4 \ge \cdots \ge \left([\PP \tilde V_{k, h + 1}](s_h^k, a_h^k)\right)^{2^{m + 1}},
        \end{align*}
         the second inequality follows from the fact that $a^x - b^x \le x \max\{a - b, 0\}$ for $a, b \in [0, 1]$ and $x\ge 1$, the third inequality follows from the monotonicity of $I_h^k$ and the definition of function $\tilde{V}_{k,h}$, the fourth inequality holds due to the definition of $R_0$ and $\check A_0$, the last inequality follows from the fact that $r(s_h^k, a_h^k) \ge 0$.
    
        Substituting \eqref{eq:rec:12} into \eqref{eq:rec:11}, we have \begin{align*} 
            \tilde S_m \le \tilde A_{m + 1} + G + 2^{m + 1}\cdot (R_0 + G + A_0 + \check A_0). 
        \end{align*}
        Thus, we complete the proof of Lemma \ref{lemma:rec:8}.
    \end{proof}
    
    \begin{lemma} \label{lemma:rec:9}
        Let $\tilde A_m$, $\tilde S_m$ be defined in \eqref{eq:def:tildeam}, \eqref{eq:def:tildesm}. 
        With probability at least $1 - 2(M + 1) \delta$, for all $m \in [M] \cup \{0\}$, \begin{align*} 
            \tilde A_m \le \sqrt{2 \tilde S_m \log(1 / \delta)}+ \frac{4}{3} \cdot \log(1 / \delta). 
        \end{align*}
        We denote the corresponding event by $\cE_{\mathrm{r}_3}$.
    \end{lemma}
    
    \begin{proof} 
        The proof is equivalent to the proof of Lemma \ref{lemma:rec:3}. 
    \end{proof}
    
    \begin{lemma} \label{lemma:rec:10}
         Let $A_m$, $\check A_m$, $R_0$, $G$ be defined in \eqref{eq:def:am}, \eqref{eq:def:checkam}, \eqref{eq:def:r0}, \eqref{eq:def:g}. On the event $\cE_{\mathrm{r}_1}\cap\cE_{\mathrm{r}_2}\cap\cE_{\mathrm{r}_3}\cap \cE_{\mathrm{var}'}\cap \cE_c$, we have \begin{align*} 
            \tilde A_1 \le 4\sqrt{(R_0 + 2G + A_0 + \check A_0) \log (1 / \delta)} + 11 \cdot \log(1 / \delta), \\
            \tilde A_0 \le 2\sqrt{(R_0 + 2G + A_0 + \check A_0) \log (1 / \delta)} + 7 \cdot \log(1 / \delta). 
        \end{align*}
    \end{lemma}
    
    \begin{proof}
        By Lemma \ref{lemma:rec:8} and Lemma \ref{lemma:rec:9}, we have for all $m \in [M] \cup\{0\}$, \begin{align}
            \tilde A_m \le \sqrt{2 \left(\tilde A_{m + 1} + 2^{m + 1}\cdot (R_0 + 2G + A_0 + \check A_0)\right) \log(1 / \delta)}+ \frac{4}{3} \cdot \log(1 / \delta). \label{eq:rec:15}
        \end{align}
        Applying Lemma \ref{lemma:rec:5}, we have \begin{align*} 
            \tilde A_1 &\le \max\left\{11 \log (1 / \delta), 4\sqrt{(R_0 + 2G + A_0 + \check A_0) \log (1 / \delta)} + 2 \log(1 / \delta)\right\}
            \\&\le  4\sqrt{(R_0 + 2G + A_0 + \check A_0) \log (1 / \delta)} + 11 \log(1 / \delta). 
        \end{align*}
        By \eqref{eq:rec:15}, it can be further deduced that \begin{align*} 
            \tilde A_0 &\le \sqrt{18 \log(1 / \delta) + 4\left(\sqrt{R_0 + 2G + A_0 + \check A_0} + \sqrt{\log(1 / \delta)}\right)^2} \sqrt{\log(1 / \delta)} + \frac{4}{3} \cdot \log(1 / \delta)
            \\&\le 2\sqrt{(R_0 + 2G + A_0 + \check A_0) \log (1 / \delta)} + 7 \cdot \log(1 / \delta). 
        \end{align*}
        Thus, we complete the proof of Lemma \ref{lemma:rec:10}.
    \end{proof}
    
    \begin{lemma} \label{lemma:rec:11}
        Let $Q_0$, $S_m$ be defined in \eqref{eq:def:q0}, \eqref{eq:def:sm}. With probability at least $1 - \delta$, it holds that \begin{align*}
            Q_0 \le 2\tilde S_0 + \tilde{O}(K). 
        \end{align*}
        We denote the corresponding event by $\cE_{\mathrm{r}_4}$. 
    \end{lemma}
    
    \begin{proof} 
        By the definition of $Q_0$, we have
        \begin{align} 
            Q_0 \le 2 \tilde S_0 + 2\sum_{k = 1}^K \sum_{h = 1}^H [\VV V_{h + 1}^{\pi_k}](s_h^k, a_h^k). \label{eq:dec:q0}
        \end{align}
        Note that for all $k \in [K]$, \begin{align} 
            \EE_{\{(s_h, a_h)\}_{h\in [H]} \sim \pi_k}\left[\sum_{h = 1}^H [\VV V_{h + 1}^{\pi_k}](s_h, a_h)\right] = \Var_{\{(s_h, a_h)\}_{h\in [H]} \sim \pi_k} \left[\sum_{h = 1}^H r(s_h, a_h) - V_{1}^{\pi_k}(s_1^k)\right] \le 1, \label{eq:cor:1}
        \end{align}
        where the last inequality holds due to the fact that $\sum_{h = 1}^H r(s_h, a_h), V_{1}^{\pi_k}(s_1^k) \in [0,1].$
       In addition, the variance is upper bounded by:
        \begin{align*} 
            \sum_{h = 1}^H \Var\left[[\VV V_{h + 1}^{\pi_k}](s_h^k, a_h^k) \big| \cF_{k, 1}\right] &\le \sum_{h = 1}^H \EE\left[\left([\VV V_{h + 1}^{\pi_k}](s_h^k, a_h^k)\right)^2 \big| \cF_{k, 1}\right]
            \\&\le \sum_{h = 1}^H 1 \cdot \EE\left[[\VV V_{h + 1}^{\pi_k}](s_h^k, a_h^k) \big| \cF_{k, 1}\right]
            \\&\le 1, 
        \end{align*}
        where the last inequality holds due to \eqref{eq:cor:1}. By Freedman's inequality (Lemma \ref{lemma:freedman}), with probability at least $1 - \delta / K$, \begin{align*}
            \sum_{h = 1}^H [\VV V_{h + 1}^{\pi_k}](s_h^k, a_h^k) \le 1 + \sqrt{2\log(K / \delta)} + 2 / 3 \cdot \log(K / \delta). 
        \end{align*}
        Using a union bound over $k \in [K]$, we can conclude that with probability at least $1 - \delta$, \begin{align} 
            \sum_{k = 1}^K \sum_{h = 1}^H [\VV V_{h + 1}^{\pi_k}](s_h^k, a_h^k) \le \tilde{O}(K). \label{eq:dec:q0:1}
        \end{align}
      Thus, we complete the proof of Lemma \ref{lemma:rec:11} by substituting \eqref{eq:dec:q0:1} into \eqref{eq:dec:q0}. 
    \end{proof}

    \subsection{Regret Analysis}
    
    \begin{proof}[Proof of Theorem \ref{thm:regret2}]
        We prove this theorem on the event $\cE_{\mathrm{r}_1} \cap \cE_{\mathrm{r}_2} \cap \cE_{\mathrm{r}_3} \cap \cE_{\mathrm{c}} \cap \cE_{\mathrm{var}'}$, which occurs with probability at least $1 - (4M + 9)\delta$ by Lemmas \ref{lemma:rec:2}, \ref{lemma:rec:3}, \ref{lemma:rec:9}, \ref{lemma:event:sub}, \ref{lemma:var'}. On these events, 
        we have the following decomposition of $\regret(K)$, \begin{align*} 
            \regret(K) &= \sum_{k = 1}^K \left[V_{1}^*(s_1^k) - V_1^{\pi_k}(s_1^k)\right]
            \\&\le \sum_{k = 1}^K \left[V_{k, 1}(s_1^k) - V_1^{\pi_k}(s_1^k)\right]
            \\&\le \sum_{k = 1}^K\sum_{h = 1}^H I_h^k \left[V_{k, h}(s_h^k) - V_{k, h + 1}(s_{h + 1}^k)\right]  - \sum_{k = 1}^K V_1^{\pi_k}(s_1^k) + G
            \\&= \sum_{k = 1}^K \sum_{h = 1}^H I_h^k \cdot r(s_h^k, a_h^k) + \sum_{k = 1}^K\sum_{h = 1}^H I_h^k \left[V_{k, h}(s_h^k) - r(s_h^k, a_h^k) - [\PP V_{k, h + 1}](s_h^k, a_h^k)\right] \\&\quad + \sum_{k = 1}^K\sum_{h = 1}^H I_h^k \left[[\PP V_{k, h + 1}](s_h^k, a_h^k) - V_{k, h + 1}(s_{h + 1}^k)\right] - \sum_{k = 1}^K V_1^{\pi_k}(s_1^k) + G
            \\&\le R_0 + A_0 + G + \underbrace{\sum_{k = 1}^K \left(\sum_{h = 1}^H  r(s_h^k, a_h^k) - V_1^{\pi_k}(s_1^k)\right)}_{I_1}, 
        \end{align*}
    where the first inequality holds due to Lemma \ref{lemma:upper:finite}, the second inequality holds due to the monotonicity of indicator function $I_h^k$, the last inequality holds due to $I_h^k\leq 1$ and $r(s_h^k,a_h^k)\ge 0$.
    
    For the term $I_1$, we have \begin{align} 
            \sum_{k = 1}^K \left(\sum_{h = 1}^H  r(s_h^k, a_h^k) - V_1^{\pi_k}(s_1^k)\right) &= \sum_{k = 1}^K \sum_{h = 1}^H \left[V_h^{\pi_k}(s_h^k) - [\PP V_{h + 1}^{\pi_k}](s_h^k, a_h^k)\right] - \sum_{k = 1}^K V_1^{\pi_k}(s_1^k) \notag
            \\&= \sum_{k = 1}^K \sum_{h = 1}^H \left[V_{h + 1}^{\pi_k}(s_{h + 1}^k) - [\PP V_{h + 1}^{\pi_k}](s_h^k, a_h^k)\right] \notag
            \\&\le |A_0| + |\check A_0| + |\tilde A_0|, \label{eq:bound:i1}
        \end{align}
        where the inequality holds due to $|x+y+z|\leq |x|+|y|+|z|$.
    
        For the term $R_0$, according to Lemma \ref{lemma:sumbell}, we have \begin{align} 
            R_0 &\le \tilde{O}\left(d\sqrt{\sum_{k = 1}^K \sum_{h = 1}^H \sigma_{h, k}^2} + d\right) \notag
            \\&\le \tilde{O}\left(d\sqrt{Q_0} + d\sqrt{\check S_0} + d\right) \notag
            \\&\le \tilde{O}\left(d\sqrt{Q_0} + d\sqrt{\check A_1 + G + R_0 + A_0} + d\right) \notag
            \\&\le \tilde{O}\left(d\sqrt{Q_0} + d\sqrt{ G + R_0 + \sqrt{Q_0} + \sqrt{R_0}} + d\right) \notag
            \\&\le \tilde{O}\left(d\sqrt{Q_0} + d^{2}\right),  \label{eq:r0}
        \end{align}where the first inequality follows from Lemma \ref{lemma:sumbell}, the second inequality follows from the definition of $\check S_0$ and $Q_0$, the third inequality holds due to Lemma \ref{lemma:rec:1}, the fourth inequality is obtained by applying Lemma \ref{lemma:rec:6} and \ref{lemma:rec:7}, the last inequality follows from the fact that $x \le a\sqrt{x} + b \Rightarrow x \le a^2 + 2b$ and the upper bound of $G$ in Lemma \ref{lemma:g}. 
    
        For the term $A_0$, by Lemma \ref{lemma:rec:7}, we have
        \begin{align}
            A_0 \le 132 \log(1 / \delta) + 28 \sqrt{R_0 \log(1 / \delta)} + 40 \sqrt{G \log(1 / \delta)} + 8\sqrt{Q_0 \log(1 / \delta)}. \notag
        \end{align}
    Putting everything together, we have \begin{align*} 
            \regret(K) \le \tilde{O}\left(d\sqrt{Q_0} + d^{2}\right). 
        \end{align*}
    Thus, we complete the proof of Theorem \ref{thm:regret2}.
    \end{proof}
    
    \begin{corollary}
    Under the same condition of Theorem \ref{thm:regret2}, with probability at least $1 - (4M + 10)\delta$, the regret of Algorithm \ref{alg:2} is bounded by: 
    \begin{align*} 
        \regret(K) \le \tilde{O}\left(d\sqrt{K} + d^2\right). 
    \end{align*}
    \end{corollary}
    
    \begin{proof} 
        We prove this corollary on the event $\cE_{\mathrm{r}_1} \cap \cE_{\mathrm{r}_2} \cap \cE_{\mathrm{r}_3} \cap \cE_{\mathrm{r}_4} \cap \cE_{\mathrm{c}} \cap \cE_{\mathrm{var}'}$, which occurs with probability at least $1 - (4M + 10)\delta$ by Lemmas \ref{lemma:rec:2}, \ref{lemma:rec:3}, \ref{lemma:rec:9}, \ref{lemma:rec:11}, \ref{lemma:event:sub}, \ref{lemma:var'}. 
    
        By  the definition of $\cE_{\mathrm{r}_4}$ in Lemma \ref{lemma:rec:11}, we have \begin{align*} 
            Q_0 &\le 2S_0 + \tilde{O}(K)
            \\&\le 2\tilde A_1 + G + 2(R_0 + G + A_0) + \tilde{O}(K)
            \\&\le 8\sqrt{(R_0 + 2G + A_0 + \check A_0) \log (1 / \delta)} + 22 \cdot \log(1 / \delta) + G + 2(R_0 + G + A_0) + \tilde{O}(K)
            \\&\le \tilde{O}\left(d\sqrt{Q_0} + d^2 + K\right)\notag\\
            &\le \tilde{O}(K + d^2), 
        \end{align*}
        where the second inequality follows from Lemma \ref{lemma:rec:8}, the third inequality holds due to \ref{lemma:rec:10}, the fourth inequality is derived by Lemma \ref{lemma:rec:7}, Lemma \ref{lemma:rec:6}, \eqref{eq:r0} and omitting the lower order terms, the last inequality holds due to the fact that $x \le a\sqrt{x} + b \Rightarrow x \le a^2 + 2b$. 
    
        By Theorem \ref{thm:regret2}, we can obtain \begin{align*} 
        \regret(K) \le \tilde{O}(d\sqrt{Q_0} + d^2) \le \tilde{O}(d\sqrt{K} + d^2). 
        \end{align*}
        Thus, we complete the proof of Corollary  \ref{coro:regret2}.
    \end{proof}

    \section{Auxiliary Lemmas}
    
    \begin{lemma}[Azuma-Hoeffding inequality, \citealt{cesa2006prediction}]\label{lemma:azuma}
        Let $\{x_i\}_{i=1}^n$ be a martingale difference sequence with respect to a filtration $\{\cG_{i}\}$ satisfying $|x_i| \leq M$ for some constant $M$, $x_i$ is $\cG_{i+1}$-measurable, $\EE[x_i|\cG_i] = 0$. Then for any $0<\delta<1$, with probability at least $1-\delta$, we have 
        \begin{align}
            \sum_{i=1}^n x_i\leq M\sqrt{2n \log (1/\delta)}.\notag
        \end{align} 
    \end{lemma}
            
    \begin{lemma}[Lemma 11,  \citealt{AbbasiYadkori2011ImprovedAF}]\label{Lemma:abba}
        For any $\lambda>0$ and sequence $\{\xb_k\}_{k=1}^K \subset \RR^d$
    for $k\in [K]$, define $\Zb_k = \lambda \Ib+ \sum_{i=1}^{k-1}\xb_i\xb_i^\top$.
    Then, provided that $\|\xb_k\|_2 \leq L$ holds for all $k\in [K]$,
    we have
    \begin{align}
        \sum_{k=1}^K \min\big\{1, \|\xb_k\|_{\Zb_{k}^{-1}}^2\big\} \leq 2d\log\big(1+KL^2/(d\lambda)\big).\notag
    \end{align}
    \end{lemma}
            
    \begin{lemma}[Lemma 12,  \citealt{AbbasiYadkori2011ImprovedAF}]\label{lemma:det}
        Suppose $\Ab, \Bb\in \RR^{d \times d}$ are two positive definite matrices satisfying that $\Ab \succeq \Bb$, then for any $\xb \in \RR^d$, $\|\xb\|_{\Ab} \leq \|\xb\|_{\Bb}\cdot \sqrt{\det(\Ab)/\det(\Bb)}$.
    \end{lemma}
    
    \begin{lemma}[\citealt{Freedman1975OnTP}] \label{lemma:freedman}
        Let $M, v > 0$ be fixed constants. Let $\{x_i\}_{i = 1}^n$ be a stochastic process, $\{\cG_i\}_i$ be a filtration so that for all $i \in [n]$, $x_i$ is $\cG_{i}$-measurable, while almost surely \begin{align*}
            \EE\left[x_i | \cG_{i - 1}\right] = 0, \quad |x_i| \le M, \quad \sum_{i = 1}^n \EE[x_i^2|\cG_{i-1}] \le v. 
        \end{align*}
        Then for any $\delta > 0$, with probability at least $1 - \delta$, we have \begin{align*} 
            \sum_{i = 1}^n x_i \le \sqrt{2v \log(1 / \delta)} + 2 / 3 \cdot M \log(1 / \delta). 
        \end{align*}
    \end{lemma}
    
    \begin{lemma}[Lemma 2, \citealt{zhang2021reinforcement}] \label{lemma:rec:5}
        Let $\lambda_1, \lambda_2, \lambda_4  > 0$, $\lambda_3 \ge 1$, and $i' = \log_2 \lambda_1$. Let $a_1, a_2, \cdots, a_{i'}$ be non-negative reals such that $a_i \le \lambda_1$ and $a_i \le \lambda_2 \sqrt{a_{i + 1} + 2^{i + 1} \lambda_3} + \lambda_4$ hold for any $1 \le i \le i'$. Then we have that \begin{align*} 
            a_1 \le \max\left\{\left(\lambda_2 + \sqrt{\lambda_2^2 + \lambda_4}\right)^2, \lambda_2 \sqrt{8 \lambda_3} + \lambda_4\right\}. 
        \end{align*}
    \end{lemma}

    \bibliographystyle{ims}
    \bibliography{ref.bib}
\end{document}